\documentclass{article}
\usepackage[english]{babel}
\usepackage[letterpaper,top=2cm,bottom=2cm,left=3cm,right=3cm,marginparwidth=1.75cm]{geometry}

\usepackage{authblk}
\usepackage{algorithm,algorithmic}

\usepackage{amsmath,amscd,amsbsy,amssymb,latexsym,url,bm,amsthm}
\usepackage{dsfont}
\usepackage{color,xcolor,natbib}
\usepackage{graphicx}

\usepackage{times}
\usepackage{threeparttable}
\usepackage{multirow}
 \usepackage{booktabs}
\usepackage{amsmath,amscd,amsbsy,amssymb,latexsym,url,bm,natbib}
\allowdisplaybreaks
\usepackage{cleveref}
\usepackage{verbatim}
\usepackage{color}
\usepackage{dsfont}
\usepackage{caption}
\usepackage{algorithm}
\usepackage{algorithmic}

\newtheorem{assumption}{Assumption}
\newtheorem{definition}{Definition}
\newtheorem{lemma}{Lemma}
\newtheorem{theorem}{Theorem}
\newtheorem{corollary}{Corollary}

\usepackage{thmtools}
\usepackage{thm-restate}
\usepackage{mathtools}

\newcommand{\EE}{\mathbb{E}}
\newcommand{\RR}{\mathbb{R}}
\newcommand{\cA}{\mathcal{A}}
\newcommand{\mBracket}[1]{\left[#1\right]}

\newcommand{\cO}{\mathcal{O}}

\newcommand{\cF}{\mathcal{F}}

\newcommand{\cS}{\mathcal{S}}

\newcommand{\cT}{\mathcal{T}}
\newcommand{\innerproduct}[2]{\langle #1, #2 \rangle}

\newcommand{\poly}{\mathrm{poly}}

\newcommand{\abs}[1]{\left| #1 \right|}

\newcommand{\bracket}[1]{\left(#1\right)}
\newcommand{\mbracket}[1]{\left[#1\right]}
\newcommand{\norm}[1]{\left\| #1 \right\|}
\newcommand{\sets}[1]{\left\{ #1 \right\}}

\newcommand{\Reg}{\mathrm{Reg}}
\newcommand{\dev}{\mathrm{DEV}}
\newcommand{\est}{\mathrm{est}}
\newcommand{\cM}{\mathcal{M}}
\newcommand{\opt}{\mathrm{opt}}
\newcommand{\op}{\mathrm{op}}

\newif\ifsup\supfalse
\suptrue

\DeclareMathOperator*{\argmin}{argmin}

\newcommand{\zxc}[1]{{\color{red}[zxc: #1]}}

\begin{document}

\title{ Improved Regret Bounds for Linear Adversarial MDPs via Linear Optimization}

\author[1]{Fang Kong}
\author[2]{Xiangcheng Zhang\footnote{Fang Kong and Xiangcheng Zhang contribute equally to this work.}}
\author[3]{Baoxiang Wang}
\author[1]{Shuai Li \thanks{Corresponding author.}}

\affil[1]{Shanghai Jiao Tong University}
\affil[2]{Tsinghua University}
\affil[3]{The Chinese University of Hong Kong, Shenzhen}

\affil[ ]{\{fangkong,shuaili8\}@sjtu.edu.cn, xc-zhang21@mails.tsinghua.edu.cn, bxiangwang@cuhk.edu.cn}
\date{}

 \maketitle

% Copyright Statement
% When submitting your final paper to a SIAM proceedings, it is requested that you include
% the appropriate copyright in the footer of the paper.  The copyright added should be
% consistent with the copyright selected on the copyright form submitted with the paper.
% Please note that "20XX" should be changed to the year of the meeting.

% % Default Copyright Statement
% \fancyfoot[R]{\scriptsize{Copyright \textcopyright\ 2023 by SIAM\\
% Unauthorized reproduction of this article is prohibited}}

\begin{abstract}%
Learning Markov decision processes (MDP) in an adversarial environment has been a challenging problem. The problem becomes even more challenging with function approximation, since the underlying structure of the loss function and transition kernel are especially hard to estimate in a varying environment.
In fact, the state-of-the-art results for linear adversarial MDP achieve a regret of $\tilde{\mathcal{O}}\bracket{K^{6/7}}$ ($K$ denotes the number of episodes), which admits a large room for improvement.
In this paper, we investigate the problem with a new view, which reduces linear MDP into linear optimization by subtly setting the feature maps of the bandit arms of linear optimization.
This new technique, under an exploratory assumption, yields an improved bound of $\tilde{\mathcal{O}}\bracket{K^{4/5}}$ for linear adversarial MDP without access to a transition simulator.
The new view could be of independent interest for solving other MDP problems that possess a linear structure.

%The adversarial MDP with linear function approximation and bandit feedback attracts lots of interest due to its potential to deal with large state and action space. 
%However, this problem still has not been theoretically understood  well, and previous algorithms for linear adversarial MDP either work with a simulator to generate a transition process or suffer from large regret.  
% However, in the linear adversarial MDP problem, 
%In this paper, we investigate a new view where the learner optimizes the linear loss vector with the estimated visitation feature of policies. \zxc{Our algorithm obtains $\tilde{\mathcal{O}}\bracket{d^{7/5}H^{12/5}K^{4/5}}$ regret bound where $d$ is the feature dimension and $H$ is the length of each episode, explicitly improving the previous state-of-the-art result of $\tilde{\mathcal{O}}\bracket{K^{6/7}}$. Moreover, when we have access to a simulator, the regret is further bounded as $\Tilde{\mathcal{O}}\bracket{\sqrt{d^2H^5K}}$.}
%We provide theoretical guarantees for two cases: when the learner has access to a simulator, the regret can be bounded by $\Tilde{\mathcal{O}}\bracket{\sqrt{d^2h^5K}}$; and when the simulator is unavailable, the regret upper bound is \fang{$O(K^{4/5})$} where $d$ is the feature dimension and $H$ is the length of each episode.  
%These results explicitly improve previous works for linear adversarial MDPs. 
% And such a view is a new approach apart from value-based and policy-based methods in the 
\end{abstract}

% \fang{DDL: 2.11 8am}

% \begin{keywords}%
% linear episodic MDP, adversarial loss, linear optimization 
% \end{keywords}
%!TEX root =  main.tex
\section{Introduction}\label{sec:intro}

Reinforcement learning (RL) describes the interaction between a learning agent and an unknown environment, where the agent aims to maximize the cumulative reward through trial and error \cite{sutton2018reinforcement}. 
%learn the unknown knowledge about the environment and maximize its cumulative rewards in a given horizon. 
It has achieved great success in many real applications, such as games \citep{mnih2013playing,silver2016mastering}, robotics \citep{kober2013reinforcement,lillicrap2015continuous}, autonomous driving \citep{kiran2021deep} and recommendation systems \citep{afsar2022reinforcement,lin2021survey}.
The interaction in RL is commonly portrayed by Markov decision processes (MDP). Most of the works study the stochastic setting, where the reward is sampled from a fixed distribution \citep{azar2017minimax,jin2018q,simchowitz2019non,yang2021q}.
%on this line pay attention to the environments with stochastic losses, where the loss suffered by the agent at a state-action pair is sampled from a fixed distribution \citep{azar2017minimax,jin2018q,simchowitz2019non,yang2021q}. 
RL in real applications is in general more challenging than the stochastic setting, as the environment could be non-stationary and the reward function could be adaptive towards the agent's policy.
For example, a scheduling algorithm will be deployed to self-interested parties, and recommendation algorithms will face strategic users.

%can suffer irregularities between episodes that can not satisfy the stochastic loss assumption, and the agent needs to adapt adversarial interventions under non-stationary environments in real world applications.

To design robust algorithms that work under non-stationary environments, a line of works focuses on the adversarial setting, where the reward function could be arbitrarily chosen by an adversary \citep{yu2009markov,rosenberg2019online,jin2020learning,chen2021minimax,luo2021policy}.  
Many works in adversarial MDPs optimize the policy by learning the value function with a tabular representation. In this case, both their computation complexity and their regret bounds depend on the state space and action space sizes. In real applications however, the state and action spaces could be exponentially large or even infinite, such as in the game of Go and in robotics. Such cost of computation and the performance are then inadequate.

%which brings the intolerable cost to computational and statistical complexity for algorithms in tabular settings. 

To cope with the curse of dimensionality,
%curse and improve learning efficiency in MDPs with large state-action spaces, 
function approximation methods are widely deployed to approximate the value functions with learnable structures. Great empirical success has proved its efficacy in a wide range of areas.
%such as the Atari and Go games \citep{mnih2013playing,silver2016mastering}. Although such methods are shown to have great empirical performances, 
Despite this, theoretical understandings of MDP with general function approximation are yet to be available.
%corresponding convergence analysis is yet hard to be derived, and researchers have focused their attention on the fundamental  linear function approximation classes.
As an essential step towards understanding function approximation, linear MDP has been an important setting and has received significant attention from the community.
It presumes that the transition and reward functions in MDP follow a linear structure with respect to a known feature \citep{jin2020provably,he2021logarithmic,hu2022nearly}.
The stochastic setting in linear MDP has been well studied and near-optimal results are available \citep{jin2020provably,hu2022nearly}. 
The adversarial setting in linear MDP is much more challenging since the underlying linear parameters of the loss function and transition kernel are especially hard to estimate in a varying environment.

%the effective exploration used in tabular settings is prohibited in such large state-action spaces, and 
%Among these, the Linear MDP is a popular linear function approximation model where both the transition and the loss function are assumed to be a linear structure between a $d$-dimensional feature and a transition/loss vector \citep{jin2020provably,he2021logarithmic,hu2022nearly}. This problem in the stochastic setting has been widely studied and the state-of-the-art approach has been shown to be near-optimal  \citep{jin2020provably,hu2022nearly}, but it is more challenging when the agent is faced with adversarial losses and bandit feedback as the agent needs to deal with changing loss functions while only with access to the loss of the experienced trajectory.

The research on linear adversarial MDPs remains open. Early works have proposed algorithms when the transition function is known. \citep{neu2021online}
%For this more challenging setting of linear adversarial MDPs, previous works only derive guarantees for the case where the agent can always call a simulator to generate a transition state and requires perfect knowledge of the transition function \citep{neu2021online}.
%under an additional condition that the agent has access to a simulator that generates the state transitions
Several recent works explore the problem without a known transition function
%and the transition is unknown, 
and derive policy optimization algorithms with the state-of-the-art regret of $\tilde{\mathcal{O}}\bracket{K^{6/7}}$ \citep{luo2021policy,dai2023refined,sherman2023improved}. 
While the optimal regret in tabular MDPs is of order $\tilde{\mathcal{O}}\bracket{{K}^{1/2}}$ \citep{jin2020learning}, the regret upper bounds available for linear adversarial MDPs seem to admit a large room for improvement.

In this paper, we investigate linear adversarial MDPs with unknown transition.
%and bandit feedback. 
We propose a new view of the problem and design an algorithm based on our view. The idea is to reduce the MDP setting to a linear optimization problem by subtly setting the feature maps of the bandit arms of linear optimization.
In this way, we operate on a set of policies and optimize the probability distribution of which policy to execute. By carefully balancing the suboptimality in policy execution, the suboptimality in policy construction, and the suboptimality in feature visitation estimation, we deduce new analyses of the problem.
%We exploit the linear properties of the linear MDP and optimize the execution probabilities of certain policies, in a way similar to optimizing the probability distribution over different arms in linear bandits.
Improved regret bounds are obtained both when we have and do not have a simulator. In particular, we conclude the first $\tilde{\mathcal{O}}\bracket{K^{4/5}}$ regret bound for linear adversarial MDPs without a simulator. 

Let $d$ be the feature dimension and $H$ be the length of each episode. Details of our contributions are as follows.
\begin{itemize}
    \item With an exploratory assumption (Assumption \ref{ass:lambda}), we obtain an $\tilde{\mathcal{O}}(d^{7/5}H^{12/5}K^{4/5})$ regret upper bound for linear adversarial MDP.
    %In the setting of linear adversarial MDP with unknown transition, but with a exploratory assumption defined in assumption \ref{ass:lambda}, we obtain regret bounded as ${\mathcal{O}}\bracket{d^{7/5}H^{12/5}K^{4/5}\log^{1/5}\frac{K}{\delta}}$. To the best of our knowledge, 
    As compared in Table \ref{table:comparison}, this is the first regret bound that achieves $\tilde{\mathcal{O}}(K^{4/5})$ order when a simulator of the transition is not provided. 
    We also want to note that our exploratory assumption that only ensures the MDP is learnable is much weaker than previous works \citep{neu2021online,luo2021policy}. Under a weaker exploratory assumption, our result achieves a significant improvement over the $\tilde{\mathcal{O}}(K^{6/7})$ regret in \citet{luo2021policy} and also removes the dependence on $\lambda$ which is the minimum eigenvalue of the exploratory policy's covariance that can be small. 
    \item In a simpler setting where the agent has access to a simulator, our regret can be further improved to $\tilde{\mathcal{O}}(\sqrt{d^2H^5K})$. This result also removes the dependence on $\lambda$
    % , the minimum eigenvalue of the exploratory policy's covariance matrices over all steps $h\in[H]$ that can be very small, 
    in previous works \citep{neu2021online,luo2021policy}. Compared with \citet{luo2021policy}, our required simulator is also weaker: we only need to have access to the trajectory when given any policy $\pi$; while \citet{luo2021policy} requires the next state $s'$ when given any state-action pair $(s,a)$. 
    % \item Under Assumption \ref{ass:lambda}, if we are instead provided with a simulator, %, in assumption \ref{ass:simulator}, 
    % we obtain a regret of $\tilde{\mathcal{O}}\bracket{\sqrt{d^2H^5K}}$. It matches the order of $K$ and improves the order of $H$ compared to existing results with a simulator, which achieved $\tilde{\mathcal{O}}\bracket{\sqrt{dH^6K }}$ \citep{dai2023refined}. Notice that \cite{dai2023refined} do not require the exploration assumption. 
    %using a linear-Q algorithm but without the exploratory assumption. 
    %Our work proposes another method of achieving $\tilde{\mathcal{O}}\bracket{\sqrt{K}}$ regret and improves the order of $H$. 
    \item Technically, we provide a new tool for linear MDP problems by exploiting the linear features of the MDP and transforming it into a linear optimization problem.
    This tool could be of independent interest and might be useful in other problems that possess a linear structure.  
    %Apart from the regret bounds obtained, we also propose a new method for linear MDP problems by exploiting the linear features of the MDP and transform it into a linear optimization problem. Thus, we can deploy the ideas of linear bandit algorithms and obtain novel regret bounds for linear MDPs, which may be of independent interest.[technique contributions] 
\end{itemize}

\begin{table}[tbh]
\centering
\begin{threeparttable}
\caption{Comparisons of our results with most related works for linear adversarial MDPs.}\label{table:comparison}
\begin{tabular}{l|l|c|c|l}
\toprule
                      & Transition               & Simulator\tnote{1}
                      & \begin{tabular}[c]{@{}l@{}}Exploratory\\ Assumption\tnote{2}\end{tabular} & Regret\tnote{3}                                                  \\ \hline
\citet{neu2021online}                   & Known                    &             yes         & yes         & $\displaystyle\tilde{\cO}(\sqrt{K/\lambda})$                                            \\ \hline
\multirow{4}{*}{\citet{luo2021policy,luo2021policyArxiv}}  & \multirow{4}{*}{Unknown} & \multirow{2}{*}{yes} & yes         & $\displaystyle\tilde{\cO}(\sqrt{K/\lambda})$ \\ \cline{4-5} 
                      &                          &                   & no          & $\displaystyle\tilde{\cO}(K^{2/3})$                                                    \\ \cline{3-5} 
                      &                          & \multirow{2}{*}{no}  & yes         & $\displaystyle\tilde{\cO}\bracket{K/\bracket{\lambda}^{2/3}}^{6/7}$                                            \\ \cline{4-5} 
                      &                          &                      & no          & $\displaystyle\tilde{\cO}(K^{14/15})$                                                  \\ \hline
\multirow{2}{*}{\citet{dai2023refined}}  & \multirow{2}{*}{Unknown} & yes                  & no          & $\displaystyle\tilde{\cO}(\sqrt{K})$                                                    \\ \cline{3-5} 
                      &                          & no                   & no          & $\displaystyle\tilde{\cO}(K^{8/9})$                                                    \\ \hline
\multirow{2}{*}{\citet{sherman2023improved}} & \multirow{2}{*}{Unknown} & yes                  & no          & $\displaystyle\tilde{\cO}(K^{2/3})$                                                    \\ \cline{3-5} 
                      &                          & no                   & no          & $\displaystyle\tilde{\cO}(K^{6/7})$                                                    \\ \hline
\multirow{2}{*}{Ours} & \multirow{2}{*}{Unknown} & yes                  & yes         & $\displaystyle\tilde{\cO}(\sqrt{K})$                                                    \\ \cline{3-5} 
                      &                          & no                   & yes         & $\displaystyle\tilde{\cO}(K^{4/5})$                                                   \\ \bottomrule
\end{tabular}
\begin{tablenotes}
\item[1] Our required simulator is defined in Assumption \ref{ass:simulator}. Notice that \citet{dai2023refined,luo2021policy,luo2021policyArxiv} adopt a stronger simulator that returns the next state $s'$ when given any state action pair $(s,a)$, while \citet{sherman2023improved,neu2021online} and this paper only need the simulator to return a trajectory when given a policy. 
\item[2] Our exploratory assumption is introduced in Assumption \ref{ass:lambda}. It is worth noting that our assumption on exploration is also much weaker than \citet{neu2021online,luo2021policy,luo2021policyArxiv}. Our exploratory assumption only ensures the learnability of the MDP while the other works require a policy that can explore the full linear space in all steps as input. Our assumption can be implied by theirs. 
\item[3] The $\lambda$ term in the regret represents the minimum eigenvalue induced by a ``good'' exploratory policy $\pi_0$, which satisfies $\lambda_{\min}\bracket{\mathbf{\Lambda}_{\pi_0,h}}\geq\lambda$ for all $h\in[H]$, where $\mathbf{\Lambda}_{\pi_0,h}$ is the covariance of $\pi_0$ at step $h$ (see Assumption \ref{ass:lambda}).
\end{tablenotes}
\end{threeparttable}
\end{table}

%!TEX root =  main.tex
\section{Related Work}\label{sec:related}

\paragraph{Linear MDPs.}
The linear function approximation problem has been studied for a long history \citep{bradtke1996linear,melo2007q,sutton2018reinforcement,yang2019sample}. 
Until recently, \citet{yang2020reinforcement} propose theoretical guarantees for the sample efficiency in the linear MDP setting. However, it assumes that the transition function can be parameterized by a small matrix. In general cases, \cite{jin2020provably} develop the first efficient algorithm LSVI-UCB both in sample and computation complexity. They show that the algorithm achieves $O(\sqrt{d^3H^3K})$ regret where $d$ is the feature dimension and $H$ is the length of each episode. 
This result is improved to the optimal order $O(dH\sqrt{K})$ by \citet{hu2022nearly} with a tighter concentration analysis. 
A very recent work \citep{he2022nearly} points out a technical error in \citet{hu2022nearly} and show a nearly minimax result that matches the lower bound $O(d\sqrt{H^3K})$ in \citet{zhou2021nearly}. 
All these works are based on UCB-type algorithms. 
Apart from UCB, the TS-type algorithm has also been proposed for this setting \citep{zanette2020frequentist}. 
And the above results mainly focus on the minimax optimality. 
In the stochastic setting, deriving an instance-dependent regret bound is also attractive as it changes in MDPs with different hardness. 
This type of regret has been widely studied under the tabular MDP setting \citep{simchowitz2019non,yang2021q}. 
\citet{he2021logarithmic} is the first to provide this type of regret in linear MDP. Using a different proof framework, they show that the LSVI-UCB algorithm can achieve $O(d^3H^5 \log K/\Delta)$ where $\Delta$ is the minimum value gap in the episodic MDP.

\paragraph{Adversarial losses in MDPs}
% When the loss function changes arbitrarily over episodes (the loss at $k$-th episode possibly depends on policies and observations in previous $k-1$ episodes), 
When the losses at state-action pairs do not follow a fixed distribution, the problem becomes the adversarial MDP. 
% Previous works in adversarial settings mainly focus on the case where the loss at $k$-th episode can be chosen arbitrarily but possibly depends on policies and observations in previous $k-1$ episodes \fang{check}. 
This problem was first studied in the tabular MDP setting. The occupancy measure-based method is one of the most popular approaches to dealing with a potential adversary. 
For this type of approach, \citet{zimin2013online} first study the known transition setting and derive regret guarantees $\Tilde{\mathcal{O}}(H\sqrt{K})$ and $\Tilde{\mathcal{O}}(\sqrt{HSAK})$ for full-information and bandit feedback, respectively. 
For the more challenging unknown transition setting, \citet{rosenberg2019online} also start from the full-information feedback and derive an $\Tilde{\mathcal{O}}(HS\sqrt{AK})$ regret. 
The bandit feedback is recently studied by \citet{jin2020learning}, where the regret bound is  $\Tilde{\mathcal{O}}(HS\sqrt{AK})$.  The other line of works \citep{neu2010online,shani2020optimistic,chen2022policy,luo2021policy} is based on policy optimization methods. In the unknown transition and bandit feedback setting, the state-of-the-art result in this line is also an $\Tilde{\mathcal{O}}(\sqrt{K})$ order achieved by \citet{luo2021policy,luo2021policyArxiv}.

Specifically, a few works focus on the linear adversarial MDP problem. \citet{neu2021online} first study the known transition setting and provide an $O(\sqrt{K})$ regret with the assumption that an exploratory policy can explore the full linear space. For the general unknown transition case, \cite{luo2021policy,luo2021policyArxiv} discuss four cases on whether a simulator is available and whether the exploratory assumption is satisfied. With the same exploratory assumption as \citet{neu2021online}, they show a regret bound $O(\sqrt{K})$ with a simulator and $O(K^{6/7})$ otherwise. 
Very recent two works \citep{dai2023refined,sherman2023improved} further generalize the setting by removing the exploratory assumption. These two works independently provide $O(K^{8/9})$ and $O(K^{6/7})$ regret for this setting when no simulator is available.

Linear mixture MDP is another popular linear function approximation model, where the transition is a mixture of linear functions.  
When considering the adversarial losses, \citet{cai2020provably,he2022near} study unknown transition but full-information feedback type, in which case the learning agent can observe the loss of all actions in each state. 
\citet{zhao2023learning} consider the general bandit feedback in this setting and show the regret in this harder environment is also $O(\sqrt{K})$. Their modeling does not assume the structure of the loss function which introduces the dependence on $S, A$ in the regret where $S$ and $A$ are the numbers of states and actions, respectively.

\section{Preliminaries}

% with unknown transition, bandit feedback, and adversarial losses, 

In this work, we study the episodic adversarial Markov Decision Processes (MDP) 
denoted by \\
$\cM(\cS, \cA, H, \sets{P_h}_{h=1}^H, \sets{\ell_k}_{k=1}^K)$ where $\cS$ is the state space, $\cA$ is the action space, $H$ is the horizon of each episode, $P_h:\cS\times \cA\times \cS \to [0,1]$ is the transition kernel of step $h$ with $P_h(s'\mid s,a)$ representing the transition probability from $s$ to $s'$ by taking action $a$ at step $h$, and $\ell_k$ is the loss function at episode $k$. 
We denote $\pi_k:\sets{\pi_{k,h}}_{h=1}^H$ as the learner's policy at each episode $k$, where $\pi_{k,h}$ is a mapping from each state to a distribution over the action space. Let $\pi_{k,h}(a\mid s)$ represent the selecting probability of action $a$ at state $s$ by following policy $\pi_k$ at step $h$.

The learner interacts with the MDP $\cM$ for $K$ episodes. At each episode $k=1,2,\ldots, K$, the environment (adversary) first chooses the loss function $\ell_k:=\sets{\ell_{k,h}}_{h=1}^H$ which may be probably based on the history information before episode $k$. The learner simultaneously decides its policy $\pi_k$. For each step $h=1,2,\ldots,H$, the learner observes the current state $s_{k,h}$, taking action $a_{k,h}$ based on $\pi_{k,h}$, and observe the loss $\ell_{k,h}(s_{k,h},a_{k,h})$. The environment will transit to the next state $s_{k,h+1}$ at the end of the step based on the transition kernel $P_h(\cdot \mid s_{k,h},a_{k,h})$.

The performance of a policy $\pi$ over episode $k$ can be evaluated by its value function, which is the expected cumulative loss, 
\begin{align*}
    V^{\pi}_k = \EE\mbracket{\sum_{h=1}^H \ell_{k,h}(s_{k,h},a_{k,h})}\,,
\end{align*}
where the expectation is taken from the randomness in the transition and the policy $\pi$.
Denote $\pi^* \in \argmin_{\pi} \sum_{k=1}^K V^{\pi}_{k} $ as the optimal policy that suffers the least expected loss over $K$ episodes. 
The objective of the learner is to minimize the cumulative regret, 
\begin{align}
    \Reg(K) = \sum_{k=1}^K\bracket{ V_{k}^{\pi_k}-V_{k}^{\pi^*} }\,,\label{eq:global:regret}
\end{align}
which is defined as the cumulative difference between the value of the taken policies and that of the optimal policy $\pi^*$.

Linear adversarial MDP denotes an MDP where both the transition kernel and the loss function are linearly depending on a feature mapping. We give a formal definition as follows.

% In this work, we are interested in the case where the state action space could be infinite, and we assume the transition kernel and loss functions satisfy a linear structure. In particular, we consider the linear MDP model defined as follows:
\begin{definition}[Linear MDP with adversarial losses]\label{definition: linear mdp}
The MDP
$\cM(\cS, \cA, H, \sets{P_h}_{h=1}^H, \sets{\ell_k}_{k=1}^K)$ is a linear MDP if there is a known feature mapping $\phi: \cS\times \cA \to \RR^d $ and unknown vector-valued measures $\mu_h\in \RR^d$ such that the transition probability at each state-action pair $(s,a)$ satisfies
\begin{equation*}
    P_h\bracket{\cdot|s,a}=\innerproduct{\phi(s,a)}{\mathbf{\mu}_h}\,.
\end{equation*}
Further, for any episode $k$ and step $h$, there exists an unknown loss vector $\theta_{k,h}\in\RR^d$ such that
\begin{align*}
   \ell_{k,h}(s,a) = \innerproduct{\phi(s,a)}{ \theta_{k,h}}\,.  
\end{align*}
for all state-action pair $(s,a)$. 
Without loss of generality, we assume $\norm{\phi\bracket{s,a}}_2\leq1$ for all $(s,a)$ and $\norm{\mu_h\bracket{\cS}}_2=\norm{\int_{s\in\cS} \mathrm{d}\mu_h\bracket{s} }_2\leq\sqrt{d}$, $\norm{\theta_{k,h}}_2\leq \sqrt{d}$ for any $k,h$.
\end{definition}

% such that for each $h\in\mBracket{H}$, there exists 

%   so that for each state-action pair $(s,a)\in \cS\times \cA$, the loss function $\ell$ and transition kernel $P$ satisfy:

Given a policy $\pi$, its feature visitation vector at step $h$ is the expected feature mapping this policy encounters at step $h$: $\phi_{\pi,h}=\EE_\pi\mBracket{\phi\bracket{s_h,a_h}}$. With this definition, the expected loss that the policy $\pi$ receives at step $h$ of episode $k$ can be written as
\begin{equation}
\ell_{k,h}^\pi:=\EE_\pi\mBracket{\ell_{k,h}\bracket{s_{k,h},a_{k,h}}}=\innerproduct{\phi_{\pi,h}}{\theta_{k,h}}\,,
\end{equation}
and the value of policy $\pi$ can be expressed as 
\begin{equation}\label{eq: value }
    V_k^\pi=\sum_{h=1}^H \ell_{k,h}=\sum_{h=1}^H\innerproduct{\phi_{\pi,h}}{\theta_{k,h}}\,.
\end{equation} 
For simplicity, we also define $\phi_{\pi,h}\bracket{s}=\EE_{a\sim\pi_h\bracket{\cdot|s}}\mBracket{\phi(s,a)}$ to represent the expected feature visitation of state $s$ at step $h$ by following $\pi$.  

To ensure that the linear MDP is learnable, we make the following exploratory assumption, which analogs to those assumptions made in previous works that study the function approximation setting \citep{neu2021online,luo2021policy,luo2021policyArxiv,pmlr-v130-hao21a,agarwal2021online}. {For any policy $\pi$, define $\mathbf{\Lambda}_{\pi,h}:=\EE_\pi\mBracket{\phi(s_h,a_h)\phi(s_h,a_h)^\top}$ as the expected covariance of $\pi$ at step $h$. Let $\lambda_{\min,h}^*=\sup_\pi \lambda_{\min}\bracket{ \mathbf{\Lambda_{\pi,h}}}$, where $\lambda_{\min}(\cdot)$ denotes the smallest eigenvalue of a matrix, and $\lambda_{\min}^*=\min_h \lambda_{\min,h}^* $. To ensure that the linear MDP is learnable, we assume that there exists a policy that generates full rank covariance matrices.
\begin{assumption}[Exploratory assumption]\label{ass:lambda}
%In our MDP, we assume  $\lambda_{\min}^*>0$.
$\lambda_{\min}^*>0$.
\end{assumption}
When the assumption is reduced to the tabular setting, where $\phi(s,a)$ is a basis vector in $\RR^{\cS\times\cA}$, this assumption becomes $\mu_{\min}:=\max_{\pi}\min_{s,a}\mu^\pi(s,a)>0$, where $\mu^\pi(s,a)$ is the probability of visiting the state-action pair $(s,a)$ under the trajectory induced by $\pi$. It simply means that there exists a policy with positive visitation probability for all state-action pairs, which is standard \citep{li2020sample}. In the linear setting, it guarantees that all the directions in $\RR^d$ is able to be visited by some policy.

We point out that this assumption is weaker than the exploratory assumptions used in previous works \citep{neu2021online,luo2021policy,luo2021policyArxiv}, in which they assume such exploratory policy $\pi_0$, satisfying $\lambda_{\min}\bracket{\mathbf{\Lambda}_{\pi_0,h}}\geq\lambda_0>0$ for all $h\in[H]$, is given as input to the algorithm. Since finding such an exploratory policy is extremely difficult, our assumption, which only requires the transition of the MDP itself to satisfy this 
constraint, is more preferred. 
}

% facing different reward functions $l_k$ selected by a  adversary at the beginning of episode $k$. At the beginning of episode $k$, the learner decides a policy $\pi_k$ and starts at the initial state $s_1$. The learner then execute $\pi_k$ and generates a trajectory $\sets{\bracket{s_{k,h},~a_{k,h},~\ell_{k,h}\bracket{s_{k,h},a_{k,h}}}}_{h=1}^H$, where $\ell_{k,h}\bracket{s_{k,h},a_{k,h}}$ is the loss received in the k-th episode by executing action $a_{k,h}$ in state $s_{k,h}$ at step $h$. \fang{We can see that from our definition above, $\EE\mBracket{\ell_{k,h}\bracket{s_{k,h},a_{k,h}}}=l_{k,h}\bracket{s_{k,h},a_{k,h}}$.} Notice that in our setting we observe no other information, which is known as bandit feedback. \\\\

% $P$ is the transition kernel and we let $P_h\bracket{\cdot|s,a}\in \Delta_\cS$ be the distribution of the next state when choosing action $a$ in state $s$ at step $h$. $\sets{\ell_s}_{s=1}^k$ is the set of loss functions in each episode, and $l_{s,h}\bracket{s,a}$ is the expected loss after executing action $a$ in state $s$ at step $h$ under loss function $\ell_s$, while the actual loss received is distributed over $\mBracket{0,~1}$.
% A policy $\pi$ is a mapping from the state space to the action space: $\cS\rightarrow\Delta_{\cA}$, where $\Delta_{\cA}$ denotes the set of distributions over $\cA$ and we let $\pi_h\bracket{a|s}$ be the probability of choosing action $a$ in state $s$ at step $h$.  \\\\

\section{Algorithm}

In this section, we introduce the proposed algorithm (Algorithm \ref{alg:adv:Hedge}). 
The algorithm takes a finite policy class $\Pi$ and the feature visitation estimators $\sets{\hat{\phi}_{\pi,h}: \pi\in\Pi,\,h\in[H]}$ as input and selects a policy $\pi_k\in \Pi$ in each episode $k\in[K]$. The acquisition of $\Pi$ and $\sets{\hat{\phi}_{\pi,h}: \pi\in\Pi,\,h\in[H]}$ will be introduced in Section \ref{sec:alg:policy} and Section \ref{sec:alg:feature}, respectively.

Recall that the loss value $\ell_{k,h}(s,a)$ is an inner product between the feature $\phi(s,a)$ and the loss vector $\theta_{k,h}$. According to this structure, we investigate ridge linear regression to estimate the unknown loss vector. To be specific, in each episode after executing policy $\pi_k$, the observed loss value can be used to estimate the loss vector $\hat{\theta}_{k,h}$ and the value function $\hat{V}_k^{\pi}$ can be estimated for each policy $\pi$ (line \ref{alg:main:estimate}). We then adopt an optimistic strategy toward the value of policies and an optimistic estimation of the policy $\pi$'s value (line \ref{alg:main:optimistic}). 
Based on the optimistic value, the exploitation probability $w(\pi)$ of a policy $\pi$ follows the EXP3-type update rule (line \ref{alg:main:prob}). To better explore each dimension of the linear space, the final selecting probability is defined as the weighted combination of the exploitation probability and an exploration probability $g(\pi)$, where the weight $\gamma$ is an input parameter (line \ref{alg:main:finalProb}). Here the exploration probability $g_h:=\sets{g_h(\pi)}_{\pi\in\Pi}$ is derived by computing the G-optimal design problem to minimize the uncertainty of all policies, i.e., 
\begin{align*}
    g_h \in \argmin_{p \in \Delta(\Pi)} \max_{\pi \in \Pi} \norm{\phi_{\pi,h}}^2_{\mathbf{V}_h(p)^{-1}}
\end{align*}
where $\mathbf{V}_h(p) = \sum_{\pi \in \Pi}p(\pi)\phi_{\pi,h} \phi_{\pi,h}^{\top}$.

If the input $\hat{\phi}_{\pi,h}$ is the true feature visitation ${\phi}_{\pi,h}$, we can ensure that the regret of the algorithm, compared with the optimal policy in $\Pi$, is upper bounded. Now it suffices to bound the additional regret caused by the sub-optimality of the best policy in $\Pi$, and the bias of the feature visitation estimators, which will be discussed in the following sections.

%The algorithm is mainly inspired by the observation that the value $V^{\pi}$ of a policy $\pi$ can be written as the inner product between the feature visitation and the loss vector thus the linear optimization technique can be deployed. 

% This procedure is inspired by \textbf{G}{\scriptsize EOMETRIC}H{\scriptsize EDGE} in  \citet{bartlett2008high} and is modified to work in the multi-step MDP setting. The idea is to exploit the linear features of the linear MDP and transform it into a linear optimization problem similar to the setting of linear bandits.

% In later \fang{section xx}, we will introduce how to construct this policy class such that the real optimal policy $\pi^*$ can be approximated by elements in $\Pi$. 

%Without loss of generality, we first consider the case where the feature visitations $\sets{\phi_{\pi,h}}_{h=1}^H$ for any policy $\pi \in \Pi$ is known beforehand. \zxc{In the general case, before we execute Algorithm \ref{alg:adv:Hedge}, we will first construct the policy set $\Pi$ using the set of feature vectors $\phi(s,a)$ given as input; then compute the feature visitation estimators for each policy $\pi\in\Pi$, which will be used in place of the accurate feature visitations. The construction and estimation process will be discussed in section \ref{sec:alg:policy} and section \ref{sec:alg:feature} respectively. }And the estimation procedure for these features will come soon. 

% Our main algorithm is a linear optimization algorithm that aims to find the optimal policy in a given $\Pi$ with known 

\begin{algorithm}[thb!]
    \caption{\textbf{G}{\scriptsize EOMETRIC}H{\scriptsize EDGE} for \textbf{L}inear \textbf{A}dversarial MDP \textbf{P}olicies (GLAP)} \label{alg:adv:Hedge}
    \begin{algorithmic}[1]
    \STATE Input: policy class $\Pi$ with feature visitation estimators $\sets{\hat{\phi}_{\pi,h}}_{h=1}^H$ for any $\pi \in \Pi$;  confidence $\delta$; exploration parameter $\gamma\leq1/2$
        \STATE Initialize: \label{alg:initial}  $\forall \pi\in\Pi, w_1\bracket{\pi}=1, W_1=|\Pi|$;  $\eta=\frac{\gamma}{dH^2}$
        \FOR{$h=1,2,\cdots,H$}
            \STATE compute the G-optimal design  $g_h(\pi)$ on the set of feature visitations: $\{\hat{\phi}_{\pi,h},\pi\in\Pi\}$. Denote $g(\pi)=\frac{1}{H}\sum_{h=1}^H g_h(\pi)$
        \label{alg: goptimal }
        \ENDFOR
        \FOR{each episode $k=1,2,\ldots, K$}
        \STATE Compute the probabilities for all policies $\pi\in\Pi$: \label{alg:main:finalProb}
        \[p_k\bracket{\pi}=\bracket{1-\gamma}\frac{w_k\bracket{\pi}}{W_k}+\gamma g(\pi)\]
        \STATE Select policy $\pi_k\sim p_k$ and observe losses $\ell_{k,h}\bracket{s_{k,h},a_{k,h}}$, for any $h\in[H]$
         \STATE Calculate the loss vector and value function estimators:\label{alg:main:estimate}
         \begin{align*}
             \hat{\theta}_{k,h}={\Sigma}_{k,h}^{-1}\hat{\phi}_{\pi_k,h}\ell_{k,h}\bracket{s_{k,h},a_{k,h}} \,,
             \hat{\ell}_{k,h}^\pi=\hat{\phi}_{\pi,h}^\top {\theta}_{k,h} \,,
             \hat{V}_k^{\pi}=\sum_{h=1}^H\hat{\ell}_{k,h}^\pi\,,
         \end{align*}
         where ${\Sigma}_{k,h}=\sum_{\pi} p_k(\pi)\hat{\phi}_{\pi,h}\hat{\phi}_{\pi,h}^{\top}$\label{alg:estimate value}
         \STATE Compute the optimistic estimate of the loss function\label{alg:main:optimistic}
            \[ \Tilde{V}_k^{\pi}=\sum_{h=1}^H\bracket{\hat{\ell}_{k,h}^\pi-2\hat{\phi}_{\pi,h}^\top\hat{\Sigma}_{k,h}^{-1}\hat{\phi}_{\pi,h}\sqrt{\frac{H\log\bracket{\frac{1}{\delta}}}{dK}}}\]
          \STATE Update the selecting probability using the loss estimators\label{alg:main:prob}
            \[ \forall \pi\in\Pi, w_{k+1}(\pi)=w_k(\pi)\exp\bracket{-\eta\Tilde{V}_k^{\pi}},
        W_{k+1}=\sum_{\pi\in\Pi}w_{k+1}(\pi)\,\]   
    \ENDFOR
    \end{algorithmic}
\end{algorithm}

% Then we construct estimators and optimistic estimators for the values of policies $\pi\in\Pi$ defined in equation \ref{eq: value }. After that we compute the probability of execution for each policy in $\Pi$ based on the potential function $w_k^\pi$, and we select the policy accordingly to minimize the regret compared with the best policy in $\Pi$.
 
% In each episode $k$, the algorithm maintains a probability distribution $\sets{p_k(\pi)}_{\pi\in\Pi}$ over $\Pi$ and would select the policy $\pi_j$ based on this distribution (Line xx).  

\subsection{Policy Construction}\label{sec:alg:policy}
In this subsection, we introduce how to construct a finite policy set $\Pi$ such that the real optimal policy $\pi^*$ can be approximated by elements in $\Pi$. The policy construction method mainly borrows from Appendix A.3 in \citet{wagenmaker2022instance} but with refined analysis for the adversarial setting.

We consider the linear softmax policy class. Specifically, given a parameter $w^{\pi} \in \RR^{d\times H}$, the induced policy $\pi$ would select action $a \in \cA$ at state $s$ with probability
\begin{align*}
\pi_h\bracket{a|s}=\frac{\exp\bracket{\eta \innerproduct{\phi\bracket{s,a}}{w_h^\pi}}}{\sum_{a'}\exp\bracket{\eta \innerproduct{\phi\bracket{s,a'}}{w_h^\pi}}} \,.
\end{align*}
The advantage of such a form of policy class is that it satisfies the Lipschitz property, i.e., the difference between values of induced policies can be upper bounded by the difference between the parameters $w$. Based on this observation, by constructing a parameter covering $\mathcal{W}$ over a $d\times H$-dimensional unit ball, we can ensure that the parameter of the optimal policy $w^*$ can be approximated by a parameter $w\in \mathcal{W}$, i.e., 
\begin{align*}
    \sum_h \norm{w_h - w_h^*}_2 \le \epsilon\,.
\end{align*}
Further, based on the Lipschitz property, the induced policy of $w$ would have a similar value to that of the optimal policy. 

The informal result is shown in the following lemma. 
\begin{lemma}
There exists a finite policy class $\Pi$ with log cardinality $\log |\Pi|=\mathcal{O}\bracket{dH^2\log K}$, such that the regret compared with the optimal policy in $\Pi$ is close to the global regret, i.e.,
\begin{align*}
    \Reg(K) := \sum_{k=1}^K \bracket{ V_k^{\pi_k} - V_k^{\pi^*} } \le \sum_{k=1}^K V_k^{\pi_k} - \min_{\pi\in \Pi}\sum_{k=1}^K V_k^{\pi}+1 := \Reg(K; \Pi)+1 \,.
\end{align*}
\end{lemma}
The detailed analysis can be found in Appendix \ref{sec:policy}.

% We consider the linear softmax policy class operating on a finite action set covering $\Tilde{\mathcal{A}}_s$. Specifically, given a parameter $w^{\pi} \in \RR^{d\times H}$, the induced policy $\pi$ would select action $a \in \cA$ at state $s$ with probability
% \begin{align*}
% \pi_h\bracket{a|s}=\frac{\exp\bracket{\eta \innerproduct{\phi\bracket{s,a}}{w_h^\pi}}\bOne{a\in\Tilde{\mathcal{A}}_s}}{\sum_{a'\in\Tilde{\mathcal{A}}_s}\exp\bracket{\eta \innerproduct{\phi\bracket{s,a'}}{w_h^\pi}}} \,.
% \end{align*}

% By constructing a covering over all potential parameters $w$, we can ensure that the induced policy set has sufficient approximation power. Specifically, the finite action set $\Tilde{\cA}_s$ is constructed using a $\delta/4H\sqrt{d}$-cover over the unit ball which contains all the feature vectors defined in Definition \ref{definition: linear mdp}. Then, for any element $\phi$ in the covering, we chose an action $a$ from the set $\sets{a\in\cA:~~\norm{\phi(s,a)-\phi}\leq \delta/2}$ and add it into the set $\Tilde{\cA}_s$. 
% Through this approach, our action covering can be understood as representatives for the action set. Then, we chose the weight vector from $w^{\pi}\in\mathcal{W}^H$, where $\mathcal{W}$ is an $\epsilon$-net of $\mathcal{B}^d\bracket{2H\sqrt{d}}$.
% Choosing the parameters $\delta$, $\eta$ and $\epsilon$ optimally as in Lemma \ref{lemma:policy}, we can bound the size of $\Pi$ and the global regret simultaneously, with $ \Reg(K)\leq  \Reg\bracket{K;\Pi}+1$ and the log cardinality of $\Pi$ bounded by $\log |\Pi|=\mathcal{O}\bracket{dH^2\log K}$. 

\subsection{Feature Visitation Estimation}
\label{sec:alg:feature}
In this subsection, we discuss how to deal with unavailable feature visitations of policies. 
Our approach is to estimate the feature visitation $\sets{\phi_{\pi,h}}_{h=1}^H$ of each policy $\pi$ and use these estimated features as input of Algorithm \ref{alg:adv:Hedge}. 
The feature estimating process is described in Algorithm \ref{alg}, which is called the feature visitation estimation oracle.

\begin{algorithm}[thb!]
    \caption{Feature visitation estimation oracle}\label{alg}
    \begin{algorithmic}[1]
    \STATE Input: policy set $\Pi$, tolerance $\epsilon\leq1/2$, confidence $\delta$ \label{alg:input}\\
    \STATE  Initialize: $\beta=16H^2\log\frac{4H^2d|\Pi|}{\delta}$, $\hat{\phi}_{\pi,1}=\EE_{a_1\sim\pi_1\bracket{\cdot|s_1}}\mBracket{\phi(s_1,a_1)}$\\
        \FOR{ $h=1,2,\dots,H-1$}
            \STATE Run procedure in Theorem \ref{theorem:9} with parameters:
            $\epsilon_{\exp}\leftarrow\frac{\epsilon^2}{d^3\beta}$, $\delta\leftarrow\frac{\delta}{2H}$, $\underline{\lambda}=\log\frac{4H^2d|\Pi|}{\delta}$, $\Phi\leftarrow\Phi_{h}:=\sets{\hat{\phi}_{\pi,h}:\pi\in\Pi}$, $\gamma_{\Phi}=\frac{1}{2\sqrt{d}}$, 
            and denote returned data as $\sets{\bracket{s_{h,\tau}, a_{h,\tau}, s_{h+1,\tau}}}_{\tau=1}^{K_{h}}$, for $K_{h}$ total number of episodes run, and covariates 
            \begin{equation*}
                \mathbf{\Lambda}_{h}\leftarrow\sum_{\tau=1}^{K_h}\phi\bracket{s_{h,\tau}, a_{h,\tau}}\phi\bracket{s_{h,\tau}, a_{h,\tau}}^\top+1/d\cdot I.
            \end{equation*}
            \label{alg:oracle}
            \FOR{$\pi\in\Pi$}
                \STATE\begin{equation*}\label{alg:return estimators}
                \hat{\phi}_{\pi, h+1}\leftarrow \bracket{\sum_{\tau=1}^{K_h}\phi_{\pi,h+1}\bracket{s_{h+1,\tau}}\phi_{h,\tau}^\top\mathbf{\Lambda}_{h}^{-1}}\hat{\phi}_{\pi,h}\,.
                \end{equation*}
            \ENDFOR
    \ENDFOR
    \RETURN  $\Phi:=\sets{\hat{\phi}_{\pi,h}: \pi\in\Pi,\,h=1,2,\cdots,H}$
    \end{algorithmic}
\end{algorithm}

For any policy $\pi$, we can first decompose its feature visitation at step $h$ as 
\begin{align*}
    \phi_{\pi,h} = \EE_{\pi} \left[ \phi(s_h,a_h) \right] 
    % =& \EE_{\pi} \left[ \EE[ \phi(s_h,a_h) \mid \cF_{h-1}]  \right] \\
    % =& \EE_{\pi} \left[ \int\int \phi(s,a) d\pi_h(a\mid s) d\mu_{h-1}(s) \phi(s_{h-1},a_{h-1})  \right]\\
    % =& \EE_{\pi} \left[ \int \phi_{\pi,h}(s)  d\mu_{h-1}(s) \phi(s_{h-1},a_{h-1})  \right] \\
    =& \int \phi_{\pi,h}(s)  d\mu_{h-1}(s) \EE_{\pi} \left[\phi(s_{h-1},a_{h-1}) \right]  \\
    :=& \cT_{\pi,h} \phi_{\pi,h-1} \\
    =& \cT_{\pi,h}\cT_{\pi,h-1}\cdots \cT_{\pi,2}\phi_{\pi,1} \,,
    % =& \cT_{\pi,h}\cT_{\pi,h-1}\cdots \cT_{\pi,2} \EE_{\pi}[\phi(s_1,a_1)]
\end{align*}
where $\cT_{\pi,h} = \int \phi_{\pi,h}(s)  d\mu_{h-1}(s)$ is the transition operator and $\phi_{\pi,1}:=\EE_{\pi}[\phi(s_1,a_1)]$ can be directly computed based on policy $\pi$. Thus, to estimate $\phi_{\pi,h}$ for each step $h$, we need to estimate the transition operator $\cT_{\pi,h}$. 

We investigate the least square method to estimate the transition operator. Consider currently we have collected $K$ trajectories, then the estimated value is given by
 \begin{align*}
       \hat{\cT} \in \argmin_{\cT} \sum_{\tau=1}^K (\cT \phi(s_{h-1,\tau}, a_{h-1,\tau}) -\sum_a \pi(a\mid s_h)\phi(s_h,a))^2 + \norm{\cT}^2 \,,
\end{align*}
and the closed-form solution is that 
\begin{align*}
        \Lambda_{h-1} &= \sum_{\tau=1}^K \phi(s_{h-1,\tau}, a_{h-1,\tau})\phi(s_{h-1,\tau}, a_{h-1,\tau})^{\top} + \lambda I\,,\\
        \hat{\cT}_{\pi,h} &=\Lambda_{h-1}^{-1} \sum_{\tau=1}^K  \phi(s_{h-1,\tau}, a_{h-1,\tau}) \phi_{\pi,h}(s_{h,\tau})\,,
\end{align*}
    where $ \phi_{\pi,h}(s_{h,\tau}) = \sum_a \pi(a\mid s_h)\phi(s_h,a)$. 

In order to guarantee the accuracy of the estimated feature visitation, we provide a guarantee for the accuracy of the estimated transition operator. The intuition is to collect enough data in each dimension of the feature space. For the design of how to collect trajectory, we adopt the reward-free technique in \citet{wagenmaker2022instance} and transform it as an independent feature visitation oracle. 
Algorithm \ref{alg} satisfies the following sample complexity and accuracy guarantees.

\begin{lemma}\label{theorem:estimate:informal}
Algorithm \ref{alg} runs for at most
%\begin{equation*}
$\tilde{\cO}\bracket{\frac{d^4H^3}{\epsilon^2}}$
%\end{equation*}
episodes and returns a feature visitation estimation that satisfies
% and outputs the policy visitation estimators $\Phi=\sets{\hat{\phi}_{\pi,h}:\,h=1,2,\cdots,H,\,\pi\in\Pi}$ with error bounded as:
\begin{equation*}
   \norm{\hat{\phi}_{\pi,h}-\phi_{\pi,h}}_2\leq\epsilon/\sqrt{d}\,,
\end{equation*}
for any policy $\pi\in\Pi$ and step $h\in[H]$, with probability at least $1-\delta$. 
\end{lemma}

% \fang{}
% In the second phase we execute Algorithm \ref{alg}, to compute the feature visitation estimators $\Phi:=\sets{\hat{\phi}_{\pi,h},\,h=1,2,\cdots,H,\,\pi\in\Pi}$. The estimators will be used in the following algorithm to estimate the value function of each policy. 

And since the regret of each episode is less than $H$, the total regret incurred in this process of estimating feature visitations is of order $\tilde{\cO}\bracket{d^4H^4/\epsilon^2}$. The detailed analysis and results are in Appendix \ref{sec:2}.

\section{Analysis}\label{sec:analysis}

This section provides the regret guarantees for the proposed algorithm as well as the proof sketch.

We consider Algorithm \ref{alg:adv:Hedge} with the policy set constructed in Section \ref{sec:alg:policy} and the feature of policies estimated in Section \ref{sec:alg:feature} as input. 
Suppose we run Algorithm \ref{alg:adv:Hedge} for $K$ rounds, the regret compared with any fixed policy $\pi\in \Pi$ in these $K$ rounds can be bounded as below.  

\begin{lemma}\label{lemma: hedge}
For any policy $\pi\in\Pi$, with probability at least $1-\delta$, 
 \begin{equation*}
     \sum_{k=1}^K \bracket{V_k^{\pi_k}-V_k^{\pi}}=\mathcal{O}\bracket{\underbrace{H\sqrt{dKH\log\frac{|\Pi|}{\delta}}+\frac{dH^2}{\gamma}\log\bracket{\frac{|\Pi|}{\delta}}+\gamma KH}_{\text{standard regret bound of EXP3}} + \underbrace{\frac{dH^2}{\gamma}\epsilon K}_{\text{feature estimation bias}}}\,,
 \end{equation*}
 where $\epsilon$ is the tolerance of the estimated feature bias in
 Lemma \ref{theorem:estimate:informal}.
\end{lemma}

Recall that when the policy set is constructed as section \ref{sec:alg:policy}, the difference between the global regret defined in Equation \eqref{eq:global:regret} and the regret compared with the policy in $\Pi$ is just a constant. So the global regret can also be bounded as above Lemma \ref{lemma: hedge}.

% We first provide the results with the given policy set $\Pi$ shown in Section \ref{sec:alg:policy}. When the estimated visitation features of policies can be estimated accurately, the regret of Algorithm \ref{alg:adv:Hedge} withing policy set $\Pi$ can be bounded as below. The term in the first bracket is similar to the regret term of EXP3-style algorithms, while the second term is due to the bias of the feature visitation estimators in lemma \ref{theorem:estimate:informal}.
% We also consider two cases whether a simulator is available as previous works \citep{luo2021policy,luo2021policyArxiv,sherman2023improved,dai2023refined}. 

Similar to previous works on linear adversarial MDP \citep{luo2021policy,luo2021policyArxiv,sherman2023improved,dai2023refined} that discuss the cases of whether a transition simulator is available, we define the simulator that may help in the following assumption. Note that this simulator is weaker than \citet{luo2021policy,luo2021policyArxiv,dai2023refined} because their simulator could generate a next state given any state-action pair.

\begin{assumption}[Simulator]\label{ass:simulator}
The learning agent has access to a simulator such that when given a policy $\pi$, it returns a trajectory based on the MDP and policy $\pi$. 
\end{assumption}

If the learning agent has access to such a simulator described in Assumption \ref{ass:simulator}, then the feature estimation process in Section \ref{sec:alg:feature} can be regarded as regret-free and the final regret is just as shown in Lemma \ref{lemma: hedge}. Otherwise, there is an additional $\tilde{\cO}(d^4H^3/\epsilon^2)$ regret term. Balancing the choice of $\gamma$ and $\epsilon$ yields the following regret upper bound.

% When the simulator is unavailable, we have to consider the additional $\tilde{\cO}\bracket{d^4H^4/\epsilon^2}$ regret occurred when we execute Algorithm 2 to obtain the feature visitation estimators. Notice that by our construction of $\Pi$ in section \ref{sec:alg:policy}, the global regret is in the same order of the regret within $\Pi$. Choosing the parameters $\gamma$ and $\epsilon$ optimally, we obtain the following results as our main theorem.

\begin{theorem}\label{theorem:simu}
With the constructed policy set in Section \ref{sec:alg:policy} and the feature estimation process in Section \ref{sec:alg:feature}, the cumulative regret satisfies
\begin{equation*}
\Reg(K)=\tilde{\mathcal{O}}\bracket{d^{7/5}H^{12/5}K^{4/5}}
\end{equation*}
with probability at least $1-\delta$. Additionally, if a simulator in Assumption \ref{ass:simulator} is accessible, the regret can be improved to
\begin{equation*}
\Reg(K)=\tilde{\mathcal{O}}\bracket{\sqrt{d^2H^5K}}
\end{equation*}
with probability at least $1-\delta$. 
\end{theorem}

The proof of the main results is deferred to Appendix \ref{sec:3}.

% When the estimated visitation features of policies satisfy  
% \begin{align*}
%     \forall \pi \in \Pi, h\in[H],  \norm{\hat{\phi}_{\pi,h}-\phi_{\pi,h}}_2\leq\epsilon/\sqrt{d}
% \end{align*}
% $\Tilde{\mathcal{O}}\bracket{d^3HK^2}$ episodes of calls to the simulator. \fang{introduce how the simulator works} \\

% The next theorem shows the regret guarantee in a more general setting where the simulator is unavailable. 

% \begin{theorem}\label{theorem:w/o simu}
% In settings where a simulator is unavailable, the regret satisfies 

% \end{theorem}

\paragraph{Discussions}

Since a main contribution of our work is to improve the results by \citet{neu2021online,luo2021policy} with an exploratory assumption, we would present more insights into the difference between our approach and the approaches in these two works.

As shown in Table \ref{table:comparison}, our result $\tilde{\cO}(K^{4/5})$ explicitly improve the result in \citet{luo2021policy} on both the dependence of $K$ and the dependence of $\lambda$. 
Recall that $\lambda$ is the minimum eigenvalue in the exploratory assumption. In real applications, for each direction in the linear space, it is reasonable that there will be a policy that visits the direction. Therefore, by mixing these policies one could ensure the exploratory assumption. However, there is no guarantee on the value of $\lambda$, and when $\lambda$ is very small the removal of the $1/\lambda$ dependence is significant.

%it is reasonable that every direction in the linear space is able to be visited by some policy. Mixing these policies together ensures that the exploratory assumption (Assumption \ref{ass:lambda}) is satisfied. The only problem is that the $\lambda$ may be small. 
%Our results remove the dependence on $\lambda$ in the main order of the regret compared with \citet{neu2021online,luo2021policy}. 
%By the way, it is also worth noting that we only require the existence of such an exploratory policy while their works need this policy as input, the computation of which may be challenging. 

% But as shown in Theorem \ref{theorem:simu}, our results do not have an explicit dependence on this parameter when omitting log terms. While \citet{neu2021online,luo2021policy} have a dependence on $1/\lambda$ in the main order. 

% \paragraph{State and action space.}
Technically, our new view of linear MDP could be general enough to be useful in other linear settings.
%achieves better results but also can apply in more general applications. 
In section \ref{sec:alg:policy}, we only introduce a simpler policy construction version to convey more intuition. We could vary this construction procedure by further putting a finite action covering over the action space to deal with the infinite action space setting. More  details can be found in Appendix \ref{sec:policy}. 
%That is to say, our approach can deal with both infinite state space and infinite action space. 
Meanwhile, \citet{neu2021online} requires both state and action spaces to be finite and \citet{luo2021policy,luo2021policyArxiv,dai2023refined,sherman2023improved} can only deal with finite action space. 

When a transition simulator is available, the number of calls to the simulator (query complexity) is also an important metric to portray the algorithm's efficiency. According to Lemma \ref{theorem:estimate:informal} and the analysis in Appendix \ref{sec:3}, we only need to call the simulator for $O(K^2)$ times to achieve an $O(\sqrt{K})$ regret (Theorem \ref{theorem:simu}) by choosing $\epsilon=O(1/K)$ in Lemma \ref{lemma: hedge}. This is much more preferred than \citet{luo2021policy}, which needs to call the simulator for $O(KAH)^{O(H)}$ times where $A$ is the action space size. 

Compared with \citet{luo2021policy}, we improve their results with better regret bounds, weaker exploratory assumption, weaker simulator and fewer queries (if we use one).

\section{Conclusion}

In this paper, we investigate linear adversarial MDP with bandit feedback. We propose a new view of linear MDP, where optimizing policies in linear MDP can be regarded as a linear optimization problem. Based on this new insight we propose an algorithm by constructing a set of policies and deploying a probability distribution over the policies to execute. With an exploratory assumption, our algorithm yields the first $\tilde\cO (K^{4/5})$ regret without access to a simulator. Compared to the results in \citet{luo2021policy}, our algorithm enjoys a weaker assumption, a better regret bound with respect to both $K$ and $\lambda$, and a weaker simulator with fewer queries if it uses one.

%In this paper, we aim to improve previous results for the linear adversarial MDP problem. Inspired by the observation that optimizing policies in linear MDP can be regarded as a linear optimization problem, we propose a new approach that optimizes over a finite policy set which can approximate the real optimal one. Our new approach achieves a significant improvement over existing works \citet{luo2021policy} in terms of the regret bound, query complexity, and generality of the action space under weaker assumptions on the exploratory policy and simulator. 

Our view contributes a new approach to linear MDP, which could be of independent interest. We demonstrated how our algorithm under this view is generalized to infinite action spaces. Future implications of this technique could involve solving other adversarial settings, such as when the loss function is corrupted up to a budget, and solving robust linear MDP where the transition kernel could change over episodes.

\bibliography{ref}
\bibliographystyle{named}

\appendix
\section{Analysis of Algorithm \ref{alg:adv:Hedge}}\label{sec:3}
In this section we propose the regret analysis for Algorithm \ref{alg:adv:Hedge} and prove the final regret bounds for our main algorithm. We will state the necessary concentration bounds as lemmas first and then analyze the regret, proving Lemma \ref{lemma: hedge} and Theorem \ref{theorem:simu}. In the following analysis, we will condition on the success of the event in Theorem \ref{theorem:estimate}, whose probability is at least $1-\delta$. The following inequality will be used in the analysis and is restated in the beginning.
\begin{lemma}\label{lem:freedman}
Let $\cF_0\subset\cF_1\subset \cdots \subset \cF_T$ be a filtration and let $X_1,X_2,\cdots X_T$ be random variables such that $X_t$ is $\cF_t$ measurable, $\EE{X_t|\cF_{t-1}}=0$,  $\abs{X_t}\leq b$ almost surely, and $\sum_{t=1}^T \EE{X_t^2|\cF_{t-1}}\leq V$ for some fixed $V>0$ and $b>0$. Then, for any $\delta\in\bracket{0,1}$, we have with probability at least $1-\delta$,
\begin{equation*}
    \sum_{t=1}^T X_t\leq 2\sqrt{V\log\bracket{1/\delta}}+b\log\bracket{1/\delta}\,.
\end{equation*}
\end{lemma}

To start with, we give the concentration of the feature visitation estimators returned from Algorithm \ref{alg}, which will be fundamental in the following analysis. Notice that $\hat{\phi}_{\pi,1}$ can be computed directly from the initial state distribution and action distribution. 
\begin{lemma}\label{lemma:estimate for main}
Conditioning on the success of the event in Theorem \ref{theorem:estimate}, we have for all episode $k$ and steps $h$, the feature visitation estimators  $\sets{\hat{\phi}_{\pi,h},\,h=1,2,\cdots,H\,,~\pi\in\Pi}$ returned by algorithm $\ref{alg}$ satisfy:
\begin{equation*}
    \abs{\innerproduct{\theta_{k,h}}{\hat{\phi}_{\pi,h}-\phi_{\pi,h}}}\leq \epsilon
\end{equation*}
\end{lemma}
\begin{proof}
Since $\norm{\theta_{k,h}}_2\leq \sqrt{d}$  for any $k$ and $h$, as in Definition \ref{definition: linear mdp},
according to Theorem \ref{theorem:estimate}, we have:
\begin{equation*}
  \abs{\innerproduct{\theta_{k,h}}{\hat{\phi}_{\pi,h}-\phi_{\pi,h}}}\leq \norm{\hat{\phi}_{\pi,h}-\phi_{\pi,h}}_2\cdot\norm{\theta_{k,h}}_2\leq\frac{\epsilon}{\sqrt{d}}\cdot \sqrt{d}\leq \epsilon\,.  
\end{equation*}
\end{proof}
The following Lemma \ref{lemma:goptimal} and Lemma \ref{lemma:eta 1} will bound the magnitude for the loss and value estimators in line \ref{alg:estimate value}, using the properties of G-optimal design computed in line \ref{alg: goptimal }.
\begin{lemma}\label{lemma:goptimal}
$\norm{\hat{\phi}_{\pi,h}}_{\hat{\Sigma}_{k,h}^{-1}}^2\leq \frac{dH}{\gamma}$ and $\abs{\hat{\phi}_{\pi,h}^\top\hat{\theta}_{k,h}}\leq \frac{dH}{\gamma}$, for all $h$ and $\pi\in\Pi$.
\end{lemma}
\begin{proof}According to the properties of G-optimal design, we have:
\begin{equation*}
    \norm{\hat{\phi}_{\pi,h}}_{\bracket{\sum_{\pi}g_h(\pi)\hat{\phi}_{\pi,h}\hat{\phi}_{\pi,h}^\top}^{-1}}^2\leq d\,,
\end{equation*}
and $\hat{\Sigma}_{k,h}\succeq \frac{\gamma}{H}\sum_{\pi}g_h(\pi)\hat{\phi}_{\pi,h}\hat{\phi}_{\pi,h}^\top$. Thus we have
$\norm{\hat{\phi}_{\pi,h}}_{\hat{\Sigma}_{k,h}^{-1}}^2\leq \frac{dH}{\gamma}$. So:
\begin{equation*}
\abs{\hat{\phi}_{\pi,h}^\top\hat{\theta}_{k,h}}=\abs{\hat{\phi}_{\pi,h}\hat{\Sigma}_{k,h}^{-1}\hat{\phi}_{\pi_k,h}\ell_{k,h}\bracket{s_{k,h},a_{k,h}}}\leq\norm{\hat{\phi}_{\pi,h}}_{\hat{\Sigma}_{k,h}^{-1}}\norm{\hat{\phi}_{\pi_k,h}}_{\hat{\Sigma}_{k,h}^{-1}}\leq \frac{dH}{\gamma}\,,~~\forall \pi\in\Pi\,.\end{equation*}
\end{proof}
\begin{lemma}\label{lemma:eta 1}
With our choice of $\eta=\frac{\gamma}{dH^2}$, when $K\geq L_0=4dH\log\bracket{\frac{|\Pi|}{\delta}}$, we have for all optimistic loss function estimator $\tilde{V}_{k}^\pi$, $\abs{\eta\tilde{V}_k^\pi}\leq 1$.
\end{lemma}
\begin{proof}
To make sure $\abs{\eta \tilde{V}_{k}^\pi}\leq 1$, we notice that:
 \begin{equation}
     \abs{\tilde{V}_{k}^\pi}\leq\sum_{h=1}^H \abs{\hat{\phi}_{\pi,h}^\top\hat{\theta}_{k,h}}+\sum_{h=1}^H 2\hat{\phi}_{\pi,h}^\top\hat{\Sigma}_{k,h}^{-1}\hat{\phi}_{k,h}\sqrt{\frac{H\log1/\delta}{dK}}\,.
 \end{equation}
 By Lemma \ref{lemma:goptimal}, we have 
 \begin{equation*}
 \abs{\tilde{V}_{k}^\pi}\leq\frac{dH^2}{\gamma}\bracket{1+2\sqrt{\frac{H\log1/\delta}{dK}}} ~. 
 \end{equation*}
 When $K\geq L_0=4dH\log\bracket{\frac{|\Pi|}{\delta}}$, we have $\abs{\tilde{V}_{k}^\pi}\leq\frac{2dH^2}{\gamma}$. Thus, our choice of $\eta=\frac{\gamma}{dH^2}$ satisfy this constraint.
\end{proof}
Throughout the following analysis, assuming we have run for some number of episodes $K$, we let $\bracket{\mathcal{F}_k}_{k=1}^K$ the filtration on this, with $\mathcal{F}_k$ the filtration up to and including episode $k$. Define $\EE_k\mbracket{\cdot}=\EE\mbracket{\cdot|\mathcal{F}_{k-1}}$.
The next lemma will bound the bias of the loss vector estimator, thus we can bound the bias of the value function estimator.
\begin{lemma}\label{lemma:theta}
Denote $\theta_{k,h}^{exp}=\EE_k\mbracket{\hat{\theta}_{k,h}}=\EE_k\mbracket{{\hat{\Sigma}}_{k,h}^{-1}\hat{\phi}_{\pi_k,h}\ell_{k,h}\bracket{s_{k,h},a_{k,h}}}$ as the expected value of the loss vector estimator on $\mathcal{F}_{k-1}$. Then we have for $\forall\pi\in\Pi$:
\begin{equation*}
    \abs{\innerproduct{\hat{\phi}_{\pi,h}}{\theta_{k,h}^{\exp}}-\innerproduct{\phi_{\pi,h}}{\theta_{k,h}}}\leq\bracket{\frac{dH}{\gamma}+1}\epsilon\leq\frac{2dH}{\gamma}\epsilon~.
\end{equation*}
As a result, we also have $\abs{\innerproduct{\hat{\phi}_{\pi,h}}{\theta_{k,h}^{\exp}}}\leq \frac{2dH}{\gamma}\epsilon+1$\, .
\end{lemma}
\begin{proof}
Using the tower rule of expectation, we have:
\begin{equation*}
    \theta_{k,h}^{exp}=\hat{\Sigma}_{k,h}^{-1}\EE_k\mbracket{\hat{\phi}_{\pi_k,h}\phi_{\pi_k,h}^\top \theta_{k,h}}=\hat{\Sigma}_{k,h}^{-1}\sum_{\pi'}p_k\bracket{\pi'}\hat{\phi}_{\pi',h}\phi_{\pi',h}^\top \theta_{k,h} ~.
\end{equation*}
Thus,
\begin{align*}
    \abs{\innerproduct{\hat{\phi}_{\pi,h}}{\theta_{k,h}^{exp}}-\innerproduct{\hat{\phi}_{\pi,h}}{\theta_{k,h}}}&=\abs{\hat{\phi}_{\pi,h}^\top\hat{\Sigma}_{k,h}^{-1}\sum_{\pi'}p_k\bracket{\pi'}\hat{\phi}_{\pi',h}\bracket{\phi_{\pi',h}^\top-\hat{\phi}_{\pi',h}^\top} \theta_{k,h}}\\
    &\leq\sum_{\pi'}p_k\bracket{\pi'}\abs{\hat{\phi}_{\pi,h}^\top\hat{\Sigma}_{k,h}^{-1}\hat{\phi}_{\pi',h}\bracket{\phi_{\pi',h}^\top-\hat{\phi}_{\pi',h}^\top} \theta_{k,h}}\\
    &\leq\sum_{\pi'}p_k\bracket{\pi'}\frac{dH}{\gamma}\epsilon=\frac{dH}{\gamma}\epsilon~,
\end{align*}
where the last inequality is due to Lemma \ref{lemma:estimate for main} and Lemma \ref{lemma:goptimal}. \\
Moreover, we also have that
  $\abs{\innerproduct{\hat{\phi}_{\pi,h}-\phi_{\pi,h}} {\theta_{k,h}}}\leq\epsilon$. Combining the two terms, we proof the lemma as:
  \begin{equation*}
      \abs{\innerproduct{\hat{\phi}_{\pi,h}}{\theta_{k,h}^{\exp}}-\innerproduct{\phi_{\pi,h}}{\theta_{k,h}}}\leq\abs{\innerproduct{\hat{\phi}_{\pi,h}}{\theta_{k,h}^{exp}}-\innerproduct{\hat{\phi}_{\pi,h}}{\theta_{k,h}}}+\abs{\innerproduct{\hat{\phi}_{\pi,h}-\phi_{\pi,h}} {\theta_{k,h}}}\leq\bracket{\frac{dH}{\gamma}+1}\epsilon~.
  \end{equation*}
\end{proof}

\begin{lemma}\label{lemma:dev}
Denote $\dev_{K,\pi}=\abs{\sum_{k=1}^K\hat{V}_k^{\pi}-V_k^{\pi}}$, then we have with probability at least $1-\delta$, 
\begin{align*}
 \dev_{K,\pi}\leq&2\sum_{k=1}^K \sum_{h=1}^H \norm{\phi_{\pi,h}}_{\Sigma_{k,h}^{-1}}^2\sqrt{\frac{H\log\bracket{\frac{1}{\delta}}}{dK}}+\frac{1}{2}\sqrt{dKH\log\bracket{\frac{1}{\delta}}}\\
 &+2\bracket{\frac{dH^2}{\gamma}}\log\bracket{\frac{1}{\delta}}+\frac{2dH^2}{\gamma}\epsilon K\,.
\end{align*}
\end{lemma}

\begin{proof}
First, we bound the bias of the estimated loss of each policy $\pi$ after episode $k$ in step $h$:
\begin{align*}
    \hat{\ell}_{k,h}^\pi-\ell_{k,h}^\pi=&\hat{\phi}_{\pi,h}^\top\hat{\theta}_{k,h}-\phi_{\pi,h}^\top\theta_{k,h}\\
    =&\bracket{\hat{\phi}_{\pi,h}^\top\hat{\theta}_{k,h}-\hat{\phi}_{\pi,h}^\top{\theta}_{k,h}^{exp}}+\bracket{\hat{\phi}_{\pi,h}^\top{\theta}_{k,h}^{exp}-\phi_{\pi,h}\theta_{k,h}}\\
    \leq& \bracket{\hat{\phi}_{\pi,h}^\top\hat{\theta}_{k,h}-\hat{\phi}_{\pi,h}^\top{\theta}_{k,h}^{exp}}+\frac{2dH}{\gamma}\epsilon\,.
\end{align*}
The first term is a martingale difference sequence as $\EE_k\mbracket{\hat{\phi}_{\pi,h}^\top\hat{\theta}_{k,h}-\hat{\phi}_{\pi,h}^\top{\theta}_{k,h}^{exp}}=0$ by the definition in Lemma \ref{lemma:theta}. To bound its magnitude, notice that $\abs{\hat{\phi}_{\pi,h}^\top\hat{\theta}_{k,h}-\hat{\phi}_{\pi,h}^\top{\theta}_{k,h}^{exp}}\leq \abs{\hat{\phi}_{\pi,h}^\top\hat{\theta}_{k,h}}+\abs{{\phi}_{\pi,h}^\top\theta_{k,h}}+\abs{\hat{\phi}_{\pi,h}^\top{\theta}_{k,h}^{exp}-{\phi}_{\pi,h}^\top\theta_{k,h}}\leq\frac{dH}{\gamma}+1+\frac{2dH}{\gamma}\epsilon\leq\frac{2dH}{\gamma}$ according to \ref{lemma:estimate for main} and \ref{lemma:goptimal}. Its variance is also bounded as:
\begin{align*}
    \EE_k\mbracket{\bracket{\hat{\phi}_{\pi,h}^\top\hat{\theta}_{k,h}-\hat{\phi}_{\pi,h}^\top{\theta}_{k,h}^{exp}}^2}\leq&\EE_k\mbracket{\bracket{\hat{\phi}_{\pi,h}^\top\hat{\theta}_{k,h}}^2}\\
    =&\EE_k\mbracket{\hat{\phi}_{\pi,h}^\top\hat{\theta}_{k,h}\hat{\theta}_{k,h}^\top\hat{\phi}_{\pi,h}}\\
    =&\EE_k\mbracket{\ell_{k,h}\bracket{s_{k,h},a_{k,h}}^2\hat{\phi}_{\pi,h}^\top\hat{\Sigma}_{k,h}^{-1}\hat{\phi}_{\pi_k,h}\hat{\phi}_{\pi_k,h}^\top\hat{\Sigma}_{k,h}^{-1}\hat{\phi}_{\pi,h} }\\
    \leq& \norm{\hat{\phi}_{\pi,h}}_{\hat{\Sigma}_{k,h}^{-1}}^2\, ,
\end{align*}
where the last inequality is due to the fact $\EE_k\mbracket{\hat{\phi}_{\pi_k,h}\hat{\phi}_{\pi_k,h}^\top}=\hat{\Sigma}_{k,h}$ and $\abs{\ell_{k,h}\bracket{s_{k,h},a_{k,h}}}\leq1$.
Using the Freedman inequality, we acquire with probability al least $1-\delta$:
\begin{align*}
    \dev_{K,\pi}\leq& \sum_{h=1}^H\abs{\sum_{k=1}^K\hat{\ell}_{k,h}^\pi-\ell_{k,h}^\pi}\\
    \leq& \sum_{h=1}^H \mbracket{ 2\sqrt{\sum_{k=1}^K\norm{\hat{\phi}_{\pi,h}}_{\hat{\Sigma}_{k,h}^{-1}}^2\log\bracket{\frac{1}{\delta}}}+\frac{2dH}{\gamma}\log\bracket{\frac{1}{\delta}}+ \frac{2dH}{\gamma}\epsilon K}\\
    \leq&2\sqrt{H\sum_{k=1}^K\sum_{h=1}^H\norm{\hat{\phi}_{\pi,h}}_{\hat{\Sigma}_{k,h}^{-1}}^2\log\bracket{\frac{1}{\delta}}}+\frac{2dH^2}{\gamma}\log\bracket{\frac{1}{\delta}}+\frac{2dH^2}{\gamma}\epsilon K\\
    \leq&2\sum_{k=1}^K \sum_{h=1}^H \norm{\hat{\phi}_{\pi,h}}_{\hat{\Sigma}_{k,h}^{-1}}^2\sqrt{\frac{H\log\bracket{\frac{1}{\delta}}}{dK}}+\frac{1}{2}\sqrt{dKH\log\bracket{\frac{1}{\delta}}}+2\frac{dH^2}{\gamma}\log\bracket{\frac{1}{\delta}}+\frac{2dH^2}{\gamma}\epsilon K\,,
\end{align*}
where the second last and last inequality is due to the Cauchy Schwartz inequality and GM-AM inequality.

\end{proof}

\begin{lemma}\label{lemma: expected bias}
We bound the gap between the actual regret and the expected estimated regret. With probability at least $1-\delta$,
\begin{equation*}
    \sum_{k=1}^K V_k^{\pi_k}-\sum_{k=1}^K\sum_\pi p_k\bracket{\pi}\tilde{V}_k^\pi \leq H\bracket{\sqrt{d}+\frac{2dH}{\gamma}\epsilon+1}\sqrt{2K\log\frac{1}{\delta}}+\frac{8dH^2}{3\gamma}\log\frac{1}{\delta}+2H\sqrt{dKH\log\frac{1}{\delta}}+\frac{2dH^2}{\gamma}\epsilon K\,.
\end{equation*}
\end{lemma}

\begin{proof}
Denote $\Bar{\phi}_{k,h}=\sum_\pi p_k\bracket{\pi}\hat{\phi}_{\pi,h}$, we have that :
\begin{align*}
    \ell_{k,h}^{\pi_k}-\sum_\pi p_k\bracket{\pi}\hat{\ell}_k^\pi=&\phi_{\pi_k,h}^\top\theta_{k,h}-\Bar{\phi}_{k,h}^\top\hat{\theta}_{k,h}\\
    =&\bracket{\phi_{\pi_k,h}^\top\theta_{k,h}-\hat{\phi}_{\pi_k,h}^\top\theta_{k,h}^{exp}}+\bracket{\hat{\phi}_{\pi_k,h}^\top\theta_{k,h}^{exp}-\Bar{\phi}_{k,h}^\top\hat{\theta}_{k,h}}\\
    \leq&\frac{2dH}{\gamma}\epsilon+ \bracket{\hat{\phi}_{\pi_k,h}^\top\theta_{k,h}^{exp}-\Bar{\phi}_{k,h}^\top\hat{\theta}_{k,h}}\,.
\end{align*}
Notice that $\EE_k\mbracket{\hat{\phi}_{\pi_k,h}^\top\theta_{k,h}^{exp}}=\EE_k\mbracket{\Bar{\phi}_{k,h}^\top\hat{\theta}_{k,h}}$. We bound its conditional variance as follows:
\begin{align}
    \sqrt{\EE_k\mbracket{\bracket{\hat{\phi}_{\pi_k,h}^\top\theta_{k,h}^{exp}-\Bar{\phi}_{k,h}^\top\hat{\theta}_{k,h}}^2}}\leq& \sqrt{\EE_k\mbracket{\bracket{\hat{\phi}_{\pi_k,h}^\top\theta_{k,h}^{exp}}}^2}+\sqrt{\EE_k\mbracket{\bracket{\Bar{\phi}_{k,h}^\top\hat{\theta}_{k,h}}^2}}\label{eq:cauchy}\\
    \leq& \frac{2dH}{\gamma}\epsilon+1+\sqrt{\Bar{\phi}_{k,h}^\top\hat{\Sigma}_{k,h}^{-1}\Bar{\phi}_{k,h}}\\
    \leq& \frac{2dH}{\gamma}\epsilon+1+\sqrt{\sum_\pi p_k\bracket{\pi}\hat{\phi}_{\pi,h}^\top\hat{\Sigma}_{k,h}^{-1}\hat{\phi}_{\pi,h}}\label{eq:jensen}\\
    =& \frac{2dH}{\gamma}\epsilon+1+\sqrt{d}\,,
\end{align}
where inequality \ref{eq:cauchy} is due to Cauchy Schwartz inequality and \ref{eq:jensen} is due to Jensen inequality. Moreover, $\abs{\hat{\phi}_{\pi_k,h}^\top\theta_{k,h}^{exp}-\Bar{\phi}_{k,h}^\top\hat{\theta}_{k,h}}\leq\frac{2dH}{\gamma}$.
Applying Bernstein's Inequality, we obtain with probability at least $1-\delta$, 
\begin{align*}
    \sum_{k=1}V_k^{\pi_k}-\sum_{k=1}^K \sum_\pi p_k\bracket{\pi}\hat{V}_k^\pi\leq&\sum_{h=1}^H \sum_{k=1}^K \bracket{\ell_{k,h}^{\pi_k}-\sum_\pi p_k\bracket{\pi}\hat{\ell}_k^\pi}\\
    \leq&  H\bracket{\sqrt{d}+\frac{2dH}{\gamma}\epsilon+1}\sqrt{2K\log\frac{1}{\delta}}+\frac{8}{3}\frac{dH^2}{\gamma}\log\frac{1}{\delta}+\frac{2dH^2}{\gamma}\epsilon K\, .
\end{align*}
Since $\hat{V}_k^\pi-\tilde{V}_k^\pi=\sum_{h=1}^H 2\hat{\phi}_{\pi,h}^\top\hat{\Sigma}_{k,h}^{-1}\hat{\phi}_{k,h}\sqrt{\frac{H\log1/\delta}{dK}}$, we have:
\begin{align*}
    \sum_{k=1}^K \sum_\pi p_k\bracket{\pi}\bracket{\hat{V}_k^\pi-\tilde{V}_k^\pi}=&\sum_{k=1}^K\sum_{h=1}^H\sum_\pi 2p_k\bracket{\pi}\hat{\phi}_{\pi,h}^\top\hat{\Sigma}_{k,h}^{-1}\hat{\phi}_{k,h}\sqrt{\frac{H\log1/\delta}{dK}}\\
    =&2H\sqrt{dKH\log1/\delta}\,.
\end{align*}
Combining the two terms, we prove this lemma.
\end{proof}

\begin{lemma}\label{lemma: square}
With probability at least $1-\delta$, we have:
\[ \sum_{k=1}^K\sum_{\pi}p_k(\pi)\bracket{\tilde{V}_k^\pi}^2\leq 2dKH^2+2\frac{dH^3}{\gamma}\sqrt{2K\log\bracket{\frac{1}{\delta}}}+\frac{8dH^3\log\bracket{\frac{1}{\delta}}}{\gamma}\,.\]
\end{lemma}
\begin{proof}
\begin{align}
    \sum_\pi p_k\bracket{\pi}\bracket{\tilde{V}_{k}^\pi}^2\leq& \sum_\pi p_k\bracket{\pi}\nonumber \mbracket{2\bracket{\hat{V}_k^\pi}^2+2\bracket{\sum_{h=1}^H2\phi_{\pi,h}^\top\Sigma_{k,h}^{-1}\phi_{\pi,h}\sqrt{H\frac{\log\bracket{\frac{1}{\delta}}}{dK}}}^2}\\\nonumber
    \leq& \sum_\pi p_k\bracket{\pi}\bracket{2\bracket{\hat{V}_k^\pi}^2+2H\sum_{h=1}^H 4 \norm{\phi_{\pi,h}}_{\Sigma_{k,h}^{-1}}^2 \phi_{\pi,h}^\top\Sigma_{k,h}^{-1}\phi_{\pi,h}
     \frac{H\log\bracket{\frac{1}{\delta}}}{dK}}\\\nonumber
     \leq&\sum_\pi p_k\bracket{\pi}\bracket{2\bracket{\hat{V}_k^\pi}^2+8\frac{dH}{\gamma}\sum_{h=1}^H\phi_{\pi,h}^\top\Sigma_{k,h}^{-1}\phi_{\pi,h}
     \frac{H\log\bracket{\frac{1}{\delta}}}{dK}}\\
     =& 2\sum_\pi p_k\bracket{\pi} \label{eq:square} \bracket{\hat{V}_k^\pi}^2+\frac{8dH^3\log1/\delta}{\gamma K}\,.
\end{align}
Since $\bracket{\hat{V}_k^\pi}^2\leq H\sum_{h=1}^H \bracket{\hat{\ell}_{k,h}^\pi}^2$, we bound the first term as follows.
\begin{equation*}
    \sum_\pi p_k\bracket{\pi}\bracket{\hat{\ell}_{k,h}^\pi}^2\leq\sum_\pi  p_k\bracket{\pi} \hat{\theta}_{k,h}^\top\hat{\phi}_{\pi,h}\hat{\phi}_{\pi,h}^\top\hat{\theta}_{k,h}\leq \hat{\phi}_{\pi_k,h}^\top\hat{\Sigma}_{k,h}^{-1}\hat{\phi}_{\pi_k,h}\,.
\end{equation*}
Its conditional expectation is $\EE_k\mbracket{\hat{\phi}_{\pi_k,h}^\top\hat{\Sigma}_{k,h}^{-1}\hat{\phi}_{\pi_k,h}}=\sum_\pi p_k\bracket{\pi}\hat{\phi}_{\pi,h}^\top\hat{\Sigma}_{k,h}^{-1}\hat{\phi}_{\pi,h}=d $, and also $\abs{\hat{\phi}_{\pi_k,h}^\top\hat{\Sigma}_{k,h}^{-1}\hat{\phi}_{\pi_k,h}}\leq\frac{dH}{\gamma}$.
Thus, applying the Hoeffding bound, we have with probability at least $1-\delta$,
\begin{align*}
    \sum_{k=1}^K\sum_\pi p_k\bracket{\pi} \bracket{\hat{V}_k^\pi}^2\leq& H\sum_{h=1}^H dK+\frac{dH}{\gamma}\sqrt{2K\log1/\delta}=dkH^2+\frac{dH^3}{\gamma}\sqrt{2K\log1/\delta}\,.
\end{align*}
Plugging it into \ref{eq:square}, we finish our proof.
\end{proof}
\begin{proof}[Proof of Lemma \ref{lemma: hedge}]
Now we are ready to start analyzing the regret. Using classical potential function analysis techniques in similar algorithms, we have:
\begin{align}
     \log\bracket{\frac{W_{K+1}}{W_1}}
     =& \sum_{k=1}^K\nonumber\nonumber \log\bracket{\frac{W_{k+1}}{W_{k}}}\\\nonumber
     =& \sum_{k=1}^K \log\bracket{\sum_\pi \frac{w_k(\pi)}{W_{k}}\exp\bracket{-\eta\tilde{V}_k^\pi}}\\\label{eq:76}
     \leq&\sum_{k=1}^K\log\bracket{\sum_\pi\frac{p_k(\pi)-\gamma g_\pi}{1-\gamma}\bracket{1-\eta\tilde{V}_k^\pi+\eta^2\bracket{\tilde{V}_k^\pi}^2 }}\\\nonumber
     \leq& \sum_{k=1}^K\sum_\pi\frac{p_k(\pi)-\gamma g_\pi}{1-\gamma}\bracket{-\eta\tilde{V}_k^\pi+\eta^2\bracket{\tilde{V}_k^\pi}^2 }\\\label{eq:rth}
     \leq& \frac{\eta}{1-\gamma}\mbracket{\sum_{k=1}^K\sum_\pi -p_k(\pi)\tilde{V}_k^\pi+\gamma\sum_{k=1}^K\sum_\pi g(\pi)\tilde{V}_k^\pi+\eta\sum_{k=1}^K\sum_{\pi}p_k(\pi)\bracket{\tilde{V}_k^\pi}^2}\,,
 \end{align}
 where inequality \ref{eq:76} is from  $\abs{\eta\tilde{V}_{k}^\pi}\leq 1$ guaranteed by Lemma \ref{lemma:eta 1}.
 Using Lemma \ref{lemma:dev}, we can bound the second term as:
 \begin{equation}\label{eq: magnitude}
 \begin{split}
     \sum_{k=1}^K \tilde{V}_k^\pi \leq& \sum_{k=1}^K V_k^\pi+\dev_{k,\pi}-\sum_{k=1}^K \sum_{h=1}^H 2\phi_{\pi,h}^\top\Sigma_{k,h}^{-1}\phi_{\pi,h}\sqrt{H\frac{\log\bracket{\frac{1}{\delta}}}{dK}}\\
     \leq& \sum_{k=1}^K V_k^\pi+\frac{1}{2}\sqrt{dKH\log\bracket{\frac{1}{\delta}}}+2\bracket{\frac{dH^2}{\gamma}}\log\bracket{\frac{1}{\delta}}+\frac{2dH^2}{\gamma}\epsilon K\\
     \leq& KH+\frac{1}{2}\sqrt{dKH\log\bracket{\frac{1}{\delta}}}+2\bracket{\frac{dH^2}{\gamma}}\log\bracket{\frac{1}{\delta}}+\frac{2dH^2}{\gamma}\epsilon K\,.
      \end{split}
 \end{equation}
 Plugging Lemma \ref{lemma: expected bias}, Lemma \ref{lemma: square} and Equation \eqref{eq: magnitude} into Equation \eqref{eq:rth}, notice we condition on $\gamma\leq 1/2$, we obtain:
  \begin{equation}\label{eq:87}
     \begin{split}
         &\log\bracket{\frac{W_{K+1}}{W_1}}\leq
         -\eta\sum_{k=1}^KV_k^{\pi_k}+2\eta^2\mbracket{2dKH^2+2\frac{dH^3}{\gamma}\sqrt{2K\log\bracket{\frac{1}{\delta}}}+\frac{8dH^3\log\bracket{\frac{1}{\delta}}}{\gamma}}\\
         &+2\eta\gamma KH+ \eta\mbracket{\frac{1}{2}\sqrt{dKH\log\bracket{\frac{1}{\delta}}}+2\bracket{\frac{dH^2}{\gamma}}\log\bracket{\frac{1}{\delta}}
         +\frac{2dH^2}{\gamma}\epsilon K}\\
         &+2\eta \mbracket{H\bracket{\sqrt{d}+\frac{2dH}{\gamma}\epsilon+1}\sqrt{2K\log\frac{1}{\delta}}+\frac{8}{3}\frac{dH^2}{\gamma}\log\frac{1}{\delta}
         +2H\sqrt{dKH\log\frac{1}{\delta}}+\frac{2dH^2}{\gamma}\epsilon K}\,.
     \end{split}
 \end{equation}
Combining terms, we have:
 \begin{equation}\label{eq:rhs}
     \begin{split}
         \frac{\log\bracket{\frac{W_{K+1}}{W_1}}}{\eta} \leq& -\sum_{k=1}^KV_k^{\pi_k}+\mathcal{O}\bracket{H\sqrt{dKH\log\frac{1}{\delta}}}+\mathcal{O}\bracket{\frac{dH^2}{\gamma}\log\bracket{\frac{1}{\delta}}+\gamma KH}\\
         &+ \mathcal{O}\bracket{\eta dKH^2 +\frac{\eta dH^3}{\gamma}\sqrt{2K\log\bracket{\frac{1}{\delta}}}+\frac{8\eta dH^3\log\bracket{\frac{1}{\delta}}}{\gamma}}+\mathcal{O}\bracket{\frac{dH^2}{\gamma}\epsilon K}\,.
     \end{split}
 \end{equation}
 Plugging $\eta=\frac{\gamma}{dH^2}$ into Equation \eqref{eq:rhs}, we have:
 \begin{equation}\label{eq: rhs super}
 \begin{split}
     \frac{\log\bracket{\frac{W_{K+1}}{W_1}}}{\eta}\leq&-\sum_{k=1}^KV_k^{\pi_k}+\mathcal{O}\bracket{H\sqrt{dKH\log\frac{1}{\delta}}}+\mathcal{O}\bracket{\frac{dH^2}{\gamma}\log\bracket{\frac{1}{\delta}}+\gamma KH}+\mathcal{O}\bracket{\frac{dH^2}{\gamma}\epsilon K}\,.
     \end{split}
 \end{equation}
 On the other hand, we have:
 \begin{equation}\label{eq:88}
 \begin{split}
      \log\bracket{\frac{W_{K+1}}{W_1}}\geq &\eta\bracket{\sum_{k=1}^K -\tilde{V}_k^\pi}-\log\bracket{|\Pi|} \\
      \geq& \eta\bracket{\sum_{k=1}^K -V_k^\pi-\dev_{K,\pi}+\sum_{k=1}^K \sum_{h=1}^H 2\phi_{\pi,h}^\top\Sigma_{k,h}^{-1}\phi_{\pi,h}\sqrt{H\frac{\log\bracket{\frac{1}{\delta}}}{dK}} }-\log\bracket{|\Pi|}\\
      \geq& \eta\bracket{\sum_{k=1}^K -V_k^\pi-\frac{1}{2}\sqrt{dKH\log\bracket{\frac{1}{\delta}}}-2\bracket{\frac{dH^2}{\gamma}}\log\bracket{\frac{1}{\delta}}-\frac{2dH^2}{\gamma}\epsilon K} -\log\bracket{|\Pi|}\,.
\end{split}
 \end{equation}
 Combining \eqref{eq: rhs super} and \eqref{eq:88}, we have:
 \begin{equation}
     \begin{split}
         \sum_{k=1}^K V_k^{\pi_k}-V_k^{\pi}\leq& \mathcal{O}\bracket{H\sqrt{dKH\log\frac{1}{\delta}}+\frac{dH^2}{\gamma}\log\bracket{\frac{1}{\delta}}+\gamma KH}+\mathcal{O}\bracket{\frac{dH^2}{\gamma}\epsilon K}+\frac{\log\bracket{|\Pi|}}{\eta}\,.
     \end{split}
 \end{equation}
 Choosing $\eta=\frac{\gamma}{dH^2}$ and combining terms, we obtain for any policy $\pi\in\Pi$, with probability at least $1-\delta$:
 \begin{equation}\label{eq:regret}
     \sum_{k=1}^K V_k^{\pi_k}-V_k^{\pi}=\mathcal{O}\bracket{H\sqrt{dKH\log\frac{|\Pi|}{\delta}}+\frac{dH^2}{\gamma}\log\bracket{\frac{|\Pi|}{\delta}}+\gamma KH}+\mathcal{O}\bracket{\frac{dH^2}{\gamma}\epsilon K}\,.
 \end{equation}
 \end{proof}
 We will then present the proof of Theorem \ref{theorem:simu} based on Equation \eqref{eq:regret}. Notice we condition on $K$ being large enough so that the optimal parameters $\gamma $ and $\epsilon$ set below are smaller than $\frac{1}{2}$, satisfying the requirements of the algorithm, while the cases of $K$ being small is trivial.
 \begin{itemize}     \item  In the case when we have access to a simulator, the total regret occurred while we execute the policies in $\Pi$. Set the parameters as $\epsilon\leftarrow {dH^2\log\frac{K}{\delta}/K}$, $\gamma\leftarrow {\sqrt{{dH\log\bracket{\frac{|\Pi|}{\delta}}}/{K}}}$ and using the properties of $\Pi$ in Lemma \ref{lemma:policy}, the total regret is bounded as:
 \begin{equation*}
     \Reg(K)\leq\Reg\bracket{K;\Pi}+1=\max_{\pi\in\Pi}\bracket{\sum_{k=1}^K V_k^{\pi_k}-V_k^{\pi}}+1=\mathcal{O}\bracket{\sqrt{d^2H^5K\log\frac{K}{\delta}}}\,.
 \end{equation*} Also, according to corollary \ref{order of complexity}, the total number of episodes run on the simulator is in the order of $\tilde{\mathcal{O}}\bracket{d^3HK^2}$.\\
 \item When we don't have access to a simulator, we have to take account of the regret occurred while we estimate the feature visitation of each policy. According to corollary \ref{order of complexity}, the additional regret is in the order of $\mathcal{O}\bracket{\frac{d^4H^4}{\epsilon^2}\log\frac{H^2d|\Pi|}{\delta}+C_1}$.
 By our construction of policy set $\Pi$ in Lemma \ref{lemma:policy}, the total regret is bounded as:
 \begin{equation}
     \Reg(K)=\mathcal{O}\bracket{\frac{d^5H^6}{\epsilon^2}\log\frac{K}{\delta}+\frac{d^2H^4}{\gamma}\log\frac{K}{\delta}+\gamma KH+\frac{dH^2}{\gamma}\epsilon K+\sqrt{d^2H^5K\log\frac{K}{\delta}}+C_1},
 \end{equation}
 with $C_1=\poly\bracket{d,H, \log1/\delta, \frac{1}{\lambda_{min}^*},\log|\Pi|, \log1/\epsilon}$.
 Set the parameters as $\epsilon\leftarrow {K^{-2/5}d^{9/5}H^{9/5}\log^{2/5}\frac{K}{\delta}}$, $\gamma\leftarrow{ K^{-1/5}d^{7/5}H^{7/5}\log^{1/5}\frac{K}{\delta}}$, the total regret is in the order of:
 \begin{equation*}
     \Reg(K)={\mathcal{O}}\bracket{d^{7/5}H^{12/5}K^{4/5}\log^{1/5}\frac{K}{\delta}}\,.
 \end{equation*}
 
 \end{itemize}

\section{Construct the Policy Visitation Estimators}\label{sec:2}

In this section, we will propose the analysis of Algorithm $\ref{alg}$. We then prove theorem $\ref{theorem:estimate}$ and corollary \ref{order of complexity} as our main results, which will provide the concentration of the estimators $\hat{\phi}_{\pi,h}$ and bound the sample complexity. These results will then be used to proof the final regret bounds in Appendix \ref{sec:3}.\\\\
First, we propose the performance guarantee of the data collecting oracle, which comes directly from theorem 9 in \citet{wagenmaker2022instance}. 
Denote:
\begin{equation*}
    \mathbf{XY}_{\opt}\bracket{\mathbf{\Lambda}}=\max_{\phi\in\Phi}\norm{\phi}_{\mathbf{A}\bracket{\mathbf{\Lambda}}^{-1}}^2~~\mathrm{for}~~\mathbf{A}\bracket{\mathbf{\Lambda}}=\mathbf{\Lambda}+\mathbf{\Lambda}_0\,,
\end{equation*}
for $\mathbf{\Lambda}_0$ be some fixed regularizer. We consider it's smooth approximation:
\begin{equation*}
    \widetilde{\mathbf{XY}}_{\opt}\bracket{\mathbf{\Lambda}}=\frac{1}{\eta}\log\bracket{\sum_{\phi\in\Phi}e^{\eta\norm{\phi}_{\mathbf{A}\bracket{\mathbf{\Lambda}}^{-1}}^2}}\,.
\end{equation*} 
We also define $\mathbf{\Omega}_h:=\sets{\EE_{\pi\sim\omega}\mBracket{\mathbf{\Lambda_{\pi,h}}}\,:~~\omega\in\Delta_\pi}$, where $\Delta_\pi$ is the set of all the distributions over all valid Markovian policies. $\mathbf{\Omega}_h$ is, then, the set of all covariance matrices realizable by distributions over policies at step $h$.Then we have
\begin{theorem}\label{theorem:9}
Considering running Algorithm 6 in \citet{wagenmaker2022instance} with some $\epsilon>0$ and functions 
\begin{equation*}
    f_i\bracket{\mathbf{\Lambda}}\leftarrow\widetilde{\mathbf{XY}}_{opt}\bracket{\mathbf{\Lambda}}
\end{equation*}
for $\mathbf{\Lambda}_0\leftarrow\bracket{T_iK_i}^{-1}\Sigma_i=:\mathbf{\Lambda}_i$ and 
\begin{align*}
    &\eta_i=\frac{2}{\gamma_\Phi}\cdot\bracket{1+\norm{\mathbf{\Lambda}_i}_{\op}}\cdot\log|\Phi|\\
    &L_i=\norm{\mathbf{\Lambda}_i^{-1}}_{\op}^2\,,~~\beta_i=2\norm{\mathbf{\Lambda}_i^{-1}}_{\op}^3\bracket{1+\eta_i\norm{\mathbf{\Lambda}_i^{-1}}_{\op}}\,,~~M_i=\norm{\mathbf{\Lambda}_i^{-1}}_{\op}^2
\end{align*}
where $\Sigma_i$ is the matrix returned by running Algorithm 7 in \citet{wagenmaker2022instance} with $N\leftarrow T_iK_i$, $\delta\leftarrow\delta/\bracket{2i^2}$, and some $\underline{\lambda}\geq0$. Then with probability $1-2\delta$, this procedure will collect at most 
\begin{equation*}
    20\cdot\frac{\inf_{\mathbf{\Lambda}\in\Omega}\max_{\phi\in\Phi}\norm{\phi}^2_{\mathbf{A}\bracket{\mathbf{\Lambda}}^{-1}}}{\epsilon_{\exp}}+\poly\bracket{d,\,H,\,\log1/\delta,\,\frac{1}{\lambda_{\min}^*},\,\frac{1}{\gamma_\Phi},\,\underline{\lambda},\,\log|\Phi|,\,\log\frac{1}{\epsilon_{\exp}}}
\end{equation*}
episodes, where 
\begin{equation*}
    \mathbf{A}\bracket{\mathbf{\Lambda}}=\mathbf{\Lambda}+\min\left\{\frac{\bracket{\lambda_{\min}^*}^2}{d},\,\frac{\lambda_{\min}^*}{d^3H^3\log^{7/2}1/\delta}\right\}\cdot\poly\log\bracket{\frac{1}{\lambda_{\min}^*}\,,~d\,,~H\,,~\underline{\lambda}\,,~\log\frac{1}{\delta}} \cdot I\,,
\end{equation*}
and will produce covariates $\widehat{\Sigma}+\Sigma_i$ such that
\begin{equation*}
    \max_{\phi\in\Phi}\norm{\phi}^2_{\bracket{\widehat{\Sigma}+\Sigma_i}^{-1}}\leq\epsilon_{\exp}
\end{equation*}
and
\begin{equation*}
    \lambda_{\min}\bracket{\widehat{\Sigma}+\Sigma_i}\geq\max\left\{d\log1/\delta\,,~\lambda\right\}\,.
\end{equation*}
\end{theorem}
Next, we will propose the concentration analysis of our estimators and bound the total number of episodes run.
Throughout this section, assuming we have run for some number of episodes K, we let $\bracket{\mathcal{F}_\tau}_{\tau=1}^K$ the filtration on this, with $\mathcal{F}_\tau$ the filtration up to and including episode $\tau$. We also let $\mathcal{F}_{\tau,h}$ denote the filtration on all episodes $\tau'<\tau$, and on steps $h'=1,\cdots,h$ of episode $\tau$.
Define 
\begin{equation*}
    \phi_{\pi,h}=\EE_\pi\mBracket{\phi\bracket{s_h,a_h}},~~\phi_{\pi,h}\bracket{s}=\sum_{a\in\mathcal{A}}\phi\bracket{s,a}\pi_h\bracket{a|s}
\end{equation*}
and
\begin{equation*}
    \mathcal{T}:=\int \phi_{\pi,h}\bracket{s}\,d\mu_{h-1}\bracket{s}^\top\,.
\end{equation*}
We have from lemma A.7 in \citet{wagenmaker2022instance}: $\phi_{\pi,h}=\mathcal{T}_{\pi,h}\phi_{\pi,h-1}=\cdots=\mathcal{T}_{\pi,h}\cdots\mathcal{T}_{\pi,1}\phi_{\pi,0}$\,.
We also denote $\gamma_\Phi:=\max_{\phi\in\Phi}\norm{\phi}_2$.\\\\
The following Lemma \ref{concentration} comes straight from lemma B.1, lemma B.2 and lemma B.3 in \citet{wagenmaker2022instance} and provides us with the basic concentration properties of the estimators constructed in line \ref{alg:return estimators} of Algorithm \ref{alg}.
\begin{lemma}\label{concentration}
Assume that we have collected some data $\sets{\bracket{s_{h-1,\tau},a_{h-1,\tau},s_{h,\tau}}}_{\tau=1}^K$ where, for each $\tau'$, $s_{h,\tau'}|\mathcal{F}_{h-1,\tau'}$ is independent of $\sets{\bracket{s_{h-1,\tau},a_{h-1,\tau},s_{h,\tau}}}_{\tau\neq\tau'}$. Denote $\phi_{h-1,\tau}=\phi\bracket{s_{h-1,\tau},a_{h-1,\tau}}$ and $\mathbf{\Lambda}_{h-1}=\sum_{\tau=1}^K \phi_{h-1,\tau}\phi_{h-1,\tau}^\top+\lambda I$. Fix $\pi$ and let
\begin{equation*}
\begin{split}
     \hat{\mathcal{T}}_{\pi,h}=&\bracket{\sum_{\tau=1}^K \phi_{\pi,h}(s_{h,\tau})\phi_{h-1,\tau}^\top}\mathbf{\Lambda}_{h-1}^{-1}\\
     \hat{\phi}_{\pi,h}=&\hat{\mathcal{T}}_{\pi,h}\hat{\mathcal{T}}_{\pi,h-1}\cdots\hat{\mathcal{T}}_{\pi,2}\hat{\mathcal{T}}_{\pi,1}\phi_{\pi,0}\,.
\end{split}
\end{equation*}
Fix $\textbf{\textit{u}}\in \mathcal{S}^{d-1}$. Then with probability at least $1-\delta$: 
\begin{equation*}
    \abs{\innerproduct{\textbf{\textit{u}}}{\phi_{\pi,h}-\hat{\phi}_{\pi,h}}}\leq \sum_{i=1}^{h-1}\bracket{2\sqrt{\log\frac{2H}{\delta}}+\frac{\log\frac{2H}{\delta}}{\sqrt{\lambda_{\min}\bracket{\mathbf{\Lambda}_i}}}+\sqrt{d\lambda}}\cdot\norm{\hat{\phi}_{\pi,i}}_{\mathbf{\Lambda}_i^{-1}}\,.
\end{equation*}
Thus, with probability at least $1-\delta$,
\begin{equation*}
    \norm{\hat{\phi}_{\pi,h}-\phi_{\pi,h}}_2\leq d\sum_{h'=1}^{h-1}\bracket{2\sqrt{\log\frac{2Hd}{\delta}}+\frac{\log\frac{2Hd}{\delta}}{\sqrt{\lambda_{\min}\bracket{\mathbf{\Lambda}_{h'}}}}+\sqrt{d\lambda}}\cdot\norm{\hat{\phi}_{\pi,h'}}_{\mathbf{\Lambda}_{h'}^{-1}}
\end{equation*}
\end{lemma}

\begin{lemma}\label{lemma:est}
Let $\varepsilon_{\est}^h$ denote the event on which, for all $\pi\in\Pi$, the feature visitation estimates returned by line \ref{alg:return estimators} satisfy:
\begin{equation*}
    \norm{\hat{\phi}_{\pi,h+1}-\phi_{\pi,h+1}}_2\leq d\sum_{h'=1}^{h-1}\bracket{3\sqrt{\log\frac{4H^2d|\Pi|}{\delta}}+\frac{\log\frac{4H^2d|\Pi|}{\delta}}{\sqrt{\lambda_{\min}\bracket{\mathbf{\Lambda}_{h'}}}}}\cdot\norm{\hat{\phi}_{\pi,h'}}_{\mathbf{\Lambda}_{h'}^{-1}}
\end{equation*}
Then
$\mathbb{P}\mBracket{\bracket{\varepsilon_{\est}^h}^c}\leq\frac{\delta}{2H}$ .\\
\end{lemma}
\begin{proof}
Following similar analysis in lemma B.5 in \citet{wagenmaker2022instance}, we also have the data collected in Theorem \ref{theorem:9} satisfy the independent requirements of Lemma \ref{concentration}. The result follows by setting $\lambda=1/d$ in Lemma \ref{concentration}.
\end{proof}
\begin{lemma}\label{lemma:exp}
Let $\varepsilon_{\exp}^h$ denote the event on which:
The total number of episodes run in line \ref{alg:oracle} is at most 
\begin{equation*}
    C\frac{d^3\inf_{\mathbf{\Lambda}\in\Omega_h}\max_{\phi\in\Phi_h}\norm{\phi}^2_{\mathbf{A}\bracket{\mathbf{\Lambda}}^{-1}}}{\epsilon^2/\beta}+\poly\bracket{d,H, \log1/\delta, \frac{1}{\lambda_{min}^*},\log|\Pi|, \log1/\epsilon}
\end{equation*}
episodes.
The covariates returned by line \ref{alg:oracle}, $\mathbf{\Lambda}_h$, satisfy:
\begin{equation}
    \max_{\phi\in\Phi_h}\norm{\phi}_{\mathbf{\Lambda}_h^{-1}}^2\leq\frac{\epsilon^2}{d^3\beta}\,, \,~~ \lambda_{\min}\bracket{\mathbf{\Lambda}_h}\geq\log\frac{4H^2d|\Pi|}{\delta}\,.
\end{equation}
Then $\mathbb{P}\mBracket{\bracket{\varepsilon_{\exp}^h}^c\cap\varepsilon_{\est}^{h-1}\cap\bracket{\cap_{i=1}^{h-1}\varepsilon_{\exp}^i}}\leq \frac{\delta}{2H}$.
\end{lemma}
\begin{proof}
By Lemma \ref{lemma:epsilon}, on the event $\varepsilon_{\est}^{h-1}\cap\bracket{\cap_{i=1}^{h-1}\varepsilon_{\exp}^i}$ we can bound $ \norm{\hat{\phi}_{\pi,h+1}-\phi_{\pi,h+1}}_2\leq\epsilon/\sqrt{d}$. Remember that we condition on $K$ being large enough so that we have $\epsilon\leq1/2$. Also, we can lower bound $\norm{\phi_{\pi,h}}_2\geq1/\sqrt{d}$ from lemma A.6 in \citet{wagenmaker2022instance}. Thus,
\begin{equation*}
    \norm{\hat{\phi}_{\pi,h}}_2\geq \norm{\phi_{\pi,h}}_2-\norm{\hat{\phi}_{\pi,h+1}-\phi_{\pi,h+1}}_2\geq1/\sqrt{d}-\epsilon/\sqrt{d}\geq 1/\bracket{2\sqrt{d}}\,.
\end{equation*}
So the choice of $\gamma_\Phi=\frac{1}{2\sqrt{d}}$ is valid. The result then follows by applying Theorem \ref{theorem:9} with our chosen parameters.
\end{proof}
\begin{lemma}\label{lemma:epsilon}
On the event $\varepsilon_{\est}^{h}\cap\bracket{\cap_{i=1}^{h}\varepsilon_{\exp}^i}$, for all $\pi\in\Pi$,
\begin{equation*}
   \norm{\hat{\phi}_{\pi,h+1}-\phi_{\pi,h+1}}_2\leq\epsilon/\sqrt{d}\,.
\end{equation*}
\end{lemma}
\begin{proof}
On $\varepsilon_{exp}^i$, we can bound:
\begin{equation*}
\begin{split}
    \lambda_{\min}\bracket{\mathbf{\Lambda}_i}&\geq\log\frac{4H^2d|\Pi|}{\delta}\,,\\
    \norm{\hat{\phi}_{\pi,i}}_{\mathbf{\Lambda}_i^{-1}}&\leq \frac{\epsilon}{d\sqrt{d\beta}}=\frac{\epsilon}{4Hd\sqrt{d\log\frac{4H^2d|\Pi|}{\delta}}}\,.
\end{split}
\end{equation*}
so that:
\begin{equation*}
\begin{split}
    \norm{\hat{\phi}_{\pi,h+1}-\phi_{\pi,h+1}}_2&\leq d\sum_{h'=1}^{h-1}\bracket{3\sqrt{\log\frac{4H^2d|\Pi|}{\delta}}+\frac{\log\frac{4H^2d|\Pi|}{\delta}}{\sqrt{\lambda_{\min}\bracket{\mathbf{\Lambda}_{h'}}}}}\cdot\norm{\hat{\phi}_{\pi,h'}}_{\mathbf{\Lambda}_{h'}^{-1}}\\
    &\leq d\sum_{i=1}^h 4\sqrt{\log\frac{4H^2d|\Pi|}{\delta}}\cdot\frac{\epsilon}{4Hd\sqrt{d\log\frac{4H^2d|\Pi|}{\delta}}}\\
    &\leq dH\cdot \frac{\epsilon}{dH\sqrt{d}}=\epsilon/\sqrt{d}\,.
    \end{split}
\end{equation*}
\end{proof}
\begin{lemma}\label{lemma:final concentration}
Define $\varepsilon_{\exp}=\cap_h\varepsilon_{\exp}^h$ and $\varepsilon_{\est}=\cap_h\varepsilon_{\est}^h$. Then $\mathbb{P}\mBracket{\varepsilon_{est}\cap\varepsilon_{exp}}\geq 1-\delta$, and on $\varepsilon_{est}\cap\varepsilon_{exp}$ , for all $h=1,2,\cdots,H-1$ and $\pi\in\Pi$, we have:
\begin{equation*}
   \norm{\hat{\phi}_{\pi,h+1}-\phi_{\pi,h+1}}_2\leq\epsilon/\sqrt{d}\,.
\end{equation*}
\end{lemma}
\begin{proof}
Obviously,
\begin{align*}
    \varepsilon_{\exp}^c\cup\varepsilon_{\est}^c=&\bigcup_{h=1}^{H-1}
\bracket{\bracket{\varepsilon_{\exp}^h}^c\cup\bracket{\varepsilon_{\est}^h}^c}\\
=&\bigcup_{h=1}^{H-1} \bracket{\varepsilon_{\exp}^h}^c \backslash\bracket{\bracket{\varepsilon_{\est}^{h-1}}^c\cup\bracket{\cup_{i=1}^{h-1}\bracket{\varepsilon_{\exp}^i}^c}}\cup\bigcup_{h=1}^H\bracket{\varepsilon_{\est}^h}^c\\
=&\bigcup_{h=1}^{H-1} \bracket{\varepsilon_{\exp}^h}^c \cap\bracket{{\varepsilon_{\est}^{h-1}}\cap\bracket{\cap_{i=1}^{h-1}{\varepsilon_{\exp}^i}}}\cup\bigcup_{h=1}^H\bracket{\varepsilon_{\est}^h}^c\,.
\end{align*}
Using Lemma \ref{lemma:est} and Lemma \ref{lemma:exp}, we can bound
\begin{align*}
    \mathbb{P}\mBracket{\varepsilon_{\est}^c\cup\varepsilon_{\exp}^c}\leq& \sum_{h=1}^{H-1}\bracket{\mathbb{P}\mBracket{\bracket{\varepsilon_{\exp}^h}^c\cap\varepsilon_{\est}^{h-1}\cap\bracket{\cap_{i=1}^{h-1}\varepsilon_{\exp}^i}}+\mathbb{P}\mBracket{\bracket{\varepsilon_{\est}^h}^c}}\\
    \leq& \sum_{h=1}^{H-1} 2\cdot \frac{\delta}{2H}\\
    \leq& \delta \,.
\end{align*}
And the inequality follows by Lemma \ref{lemma:epsilon}.
\end{proof}

% \thmestimate*

% \begin{restatable}[Full version of Lemma \ref{theorem:estimate:informal}]{thm}{thmestimate}\label{theorem:estimate}
% With probability at least $1-\delta$, Algorithm \ref{alg} will run at most
% \begin{equation*}
%     CH^2d^3\sum_{h=1}^{H-1} \frac{\inf_{\mathbf{\Lambda}\in\Omega_h}\max_{\pi\in\Pi} \norm{{\phi}_{\pi,h}}^2_{\mathbf{\Lambda}^{-1}}}{\epsilon^2}\log\frac{H^2d|\Pi|}{\delta}+\poly\bracket{d,H, \log1/\delta, \frac{1}{\lambda_{min}^*},\log|\Pi|, \log1/\epsilon}
% \end{equation*}
% episodes, and will output policy visitation estimators $\Phi=\sets{\hat{\phi}_{\pi,h}:\,h=1,2,\cdots,H,\,\pi\in\Pi}$ with bias bounded as:
% \begin{equation*}
%    \norm{\hat{\phi}_{\pi,h}-\phi_{\pi,h}}_2\leq\epsilon/\sqrt{d}\,.
% \end{equation*}
% \end{restatable}

\begin{lemma}[Full version of Lemma \ref{theorem:estimate:informal}]\label{theorem:estimate}
With probability at least $1-\delta$, Algorithm \ref{alg} will run at most
\begin{equation*}
    CH^2d^3\sum_{h=1}^{H-1} \frac{\inf_{\mathbf{\Lambda}\in\Omega_h}\max_{\pi\in\Pi} \norm{{\phi}_{\pi,h}}^2_{\mathbf{\Lambda}^{-1}}}{\epsilon^2}\log\frac{H^2d|\Pi|}{\delta}+\poly\bracket{d,H, \log1/\delta, \frac{1}{\lambda_{min}^*},\log|\Pi|, \log1/\epsilon}
\end{equation*}
episodes, and will output policy visitation estimators $\Phi=\sets{\hat{\phi}_{\pi,h}:\,h=1,2,\cdots,H,\,\pi\in\Pi}$ with bias bounded as:
\begin{equation*}
   \norm{\hat{\phi}_{\pi,h}-\phi_{\pi,h}}_2\leq\epsilon/\sqrt{d}\,.
\end{equation*}
\end{lemma}

\begin{proof}
According to Lemma \ref{lemma:final concentration}, we can condition on the event $\varepsilon_{est}\cap\varepsilon_{exp}$, thus we obtain the accuracy desired.
According to Lemma \ref{lemma:exp}, we have total episodes be bounded as :
\begin{align*}
    &\sum_{h=1}^{H-1} C\frac{d^3\inf_{\mathbf{\Lambda}\in\Omega_h}\max_{\phi\in\Phi_h}\norm{\phi}^2_{\mathbf{A}\bracket{\mathbf{\Lambda}}^{-1}}}{\epsilon^2/\beta}+\poly\bracket{d,H, \log1/\delta, \frac{1}{\lambda_{min}^*},\log|\Pi|, \log1/\epsilon}\\
    \leq& \sum_{h=1}^{H-1} C\frac{d^3\inf_{\mathbf{\Lambda}\in\Omega_h}\max_{\phi\in\Phi_h}\norm{\phi}^2_{\mathbf{A}\bracket{\mathbf{\Lambda}}^{-1}}}{\epsilon^2}H^2 \log\frac{H^2d|\Pi|}{\delta}+\poly\bracket{d,H, \log1/\delta, \frac{1}{\lambda_{min}^*},\log|\Pi|, \log1/\epsilon}\,.
\end{align*}
Conditioning on $\varepsilon_{est}\cap\varepsilon_{exp}$, we have for all $\pi\in\Pi$, $\norm{\hat{\phi}_{\pi,h}-\phi_{\pi,h}}_2\leq\epsilon/\sqrt{d}$, thus we can upper bound:
\begin{align*}
    \inf_{\mathbf{\Lambda}\in\Omega_h}\max_{\phi\in\Phi_h}\norm{\phi}^2_{\mathbf{A}\bracket{\mathbf{\Lambda}}^{-1}}=&\inf_{\mathbf{\Lambda}\in\Omega_h}\max_{\pi\in\Pi}\norm{\hat{\phi}_{\pi,h}}^2_{\mathbf{A}\bracket{\mathbf{\Lambda}}^{-1}}\\
    \leq& \inf_{\mathbf{\Lambda}\in\Omega_h}\max_{\pi\in\Pi} \bracket{2\norm{{\phi}_{\pi,h}}^2_{\mathbf{A}\bracket{\mathbf{\Lambda}}^{-1}}+2\norm{\hat{\phi}_{\pi,h}-\phi_{\pi,h}}^2_{\mathbf{A}\bracket{\mathbf{\Lambda}}^{-1}}}\\
    \leq& \inf_{\mathbf{\Lambda}\in\Omega_h}\max_{\pi\in\Pi}\bracket{2\norm{{\phi}_{\pi,h}}^2_{\mathbf{A}\bracket{\mathbf{\Lambda}}^{-1}}+\frac{2\epsilon^2}{d\lambda_{min}\bracket{\mathbf{A}\bracket{\mathbf{\Lambda}}}}}\\
    \leq& \inf_{\mathbf{\Lambda}\in\Omega_h}\max_{\pi\in\Pi} 2\norm{{\phi}_{\pi,h}}^2_{\mathbf{\Lambda}^{-1}}+ \frac{2\epsilon^2}{d\lambda_{min}^*}
\end{align*}
Thus, the total number of episodes is bounded as:
\begin{equation}
    CH^2d^3\sum_{h=1}^{H-1} \frac{\inf_{\mathbf{\Lambda}\in\Omega_h}\max_{\pi\in\Pi} \norm{{\phi}_{\pi,h}}^2_{\mathbf{\Lambda}^{-1}}}{\epsilon^2}\log\frac{H^2d|\Pi|}{\delta}+\poly\bracket{d,H, \log1/\delta, \frac{1}{\lambda_{min}^*},\log|\Pi|, \log1/\epsilon}\,.
\end{equation}
\end{proof}
From lemma B.10 in \citet{wagenmaker2022instance}, we can bound:
\begin{equation*}
    \inf_{\mathbf{\Lambda}\in\Omega_h}\max_{\pi\in\Pi} \norm{{\phi}_{\pi,h}}^2_{\mathbf{\Lambda}^{-1}}\leq d\,.
\end{equation*} 
Thus we have:
\begin{corollary}\label{order of complexity}
The sample complexity in Algorithm $\ref{alg}$ is bounded by:
\begin{equation*}
    \mathcal{O}\bracket{\frac{d^4H^3}{\epsilon^2}\log\frac{H^2d|\Pi|}{\delta}+C_1}\,,
\end{equation*}
where $C_1=\poly\bracket{d,H, \log1/\delta, \frac{1}{\lambda_{min}^*},\log|\Pi|, \log1/\epsilon}$.
\end{corollary}

\section{Construct the Policy Set}\label{sec:policy}
In this section we provide the proof for the policy set $\Pi$ we constructed. The construction techniques follows directly from Appendix A.3 in \citet{wagenmaker2022instance} and we will prove such construction also works in MDP with adversarial rewards. Our main result is stated in Lemma \ref{lemma:policy}.

\begin{lemma}\label{lemma:policy aproximation}
In the adversarial MDP setting, where the loss function changes in each round, the best policy of the MDP $\cM(\cS, \cA, H, \sets{P_h}_{h=1}^H, \sets{\ell_k}_{k=1}^K)$ in rounds $1$ to $K$ from the set of all stationary policies, is the optimal policy of the MDP with a \textbf{fixed} loss function being the average. Denote the average MDP as $\mathring{\cM}(\cS, \cA, H,\sets{P_h}_{h=1}^H, \mathring{\ell})$, with the same transition kernel and the average loss $\mathring{\ell}=\frac{1}{K} \sum_{k=1}^K \ell_k$. \\
That is:
\begin{align*}
    &\mathrm{if}~~\pi^*=\argmin_{\pi} \sum_{k=1}^K V_k^\pi\,,\\
    &\mathrm{then}~~\pi^*=\argmin_{\pi} \mathring{V}^{\pi}\,.
\end{align*}
Where $\mathring{V}$ is the value function associated with the new MDP $\mathring{\cM}$.
\end{lemma}
\begin{proof}
Let $\tau_\pi=\bracket{(s_1,a_1),(s_2,a_2),\cdots (s_H,a_H)}$ be the trajectory generated by following policy $\pi$ through the MDP. Denote the occupancy measure $\mu_h^\pi(s_h,a_h)$ as the probability of visiting state-action pair $\bracket{s_h,a_h}$ under trajectory $\tau_\pi$, and $\mu^\pi=\bracket{\mu_1^\pi,\mu_2^\pi,\cdots \mu_H^\pi}$. \\ 
For any stationary policy $\pi$, we have:
\begin{equation*}
    V^\pi =\sum_{h=1}^H\sum_{(s_h,a_h)\in\cS_h\times\cA_h}\mu_h^\pi(s_h,a_h)\ell_{k}(s_h,a_h)\,.
\end{equation*}
Since the two MDP share the same transition kernel, the occupancy measure generated by the same policy stays unchanged. So we have:
\begin{equation*}
\begin{split}
    \sum_{k=1}^K V_k^\pi&=\sum_{k=1}^K\sum_{h=1}^H\sum_{(s_h,a_h)\in\cS_h\times\cA_h}\mu_h^\pi(s_h,a_h)\ell_{k}(s_h,a_h)\\
    &=\sum_{h=1}^H\sum_{(s_h,a_h)\in\cS_h\times\cA_h} \mu_h^\pi\bracket{s_h,a_h}\bracket{\sum_{k=1}^K l_{k,h}(s_h,a_h)}\\
    &=\sum_{h=1}^H \sum_{(s_h,a_h)\in\cS_h\times\cA_h}\mu_h^\pi(s_h,a_h)\cdot K \mathring{l}(s_h,a_h)\\
    &=K \mathring{V}^\pi
\end{split}
\end{equation*}
So $\pi^*$ satisfies:
\begin{align*}
    \pi^*=\argmin_{\pi} \sum_{k=1}^K V_k^\pi
    =\argmin_{\pi} \mathring{V}^{\pi}\,.
\end{align*}
\end{proof}

\begin{lemma}\label{lemma:policy}
Choose $c$ to be an arbitrary constant, then we can construct a policy set $\Pi$ for any linear adversarial MDP $\cM(\cS, \cA, H, \sets{P_h}_{h=1}^H, \sets{\ell_k}_{k=1}^K)$, such that there exists a policy $\pi\in\Pi$, when compared with the global optimal policy, the regret of which is bounded by $1$:
\begin{equation*}
    \sum_{s=1}^K V_k^{\pi}-V_k^{\pi^*} \leq 1\,.
\end{equation*}
So that:
\begin{equation*}
    \Reg(K)=\sum_{s=1}^K V_k^{\pi_k}-V_k^{\pi^*}= \max_{\pi\in\Pi}\sum_{s=1}^K V_k^\pi-\sum_{k=1}^K V_k^{\pi^*}+ \Reg\bracket{K;\Pi}\leq  \Reg\bracket{K;\Pi}+1\,.
\end{equation*}
and the size of $\Pi$ is bounded as:
\begin{equation*}
    |\Pi|\leq \bracket{1+32K^2H^4d^{5/2}\log\bracket{1+16HdK}}^{dH^2}\,,
\end{equation*}
where $d$ is the dimension of the feature map.
\end{lemma}
\begin{proof}
According to Lemma A.14 in \citet{wagenmaker2022instance} that for any linear MDP $\mathring{\cM}(\cS, \cA, H,\sets{P_h}_{h=1}^H, \mathring{\ell})$ with fixed reward function $\mathring{\ell}$, we can construct a policy set, that there exists a policy $\pi\in\Pi$, which approximates the best policy of $\mathring{\cM}$ with bias $\abs{\mathring{V}^\pi-\mathring{V}^*}\leq \epsilon'$. And the size of the policy set is bounded as:
\begin{align}\label{eq:size}
    |\Pi|\leq \bracket{1+\frac{32H^4d^{5/2}\log\bracket{1+16Hd/\epsilon'}}{\bracket{\epsilon'}^2}}^{dH^2}\,.
\end{align}
Notice this construction is based entirely on the set of state action features $\phi\bracket{s,a}$ and require no information on the loss or reward function. In the adversarial case, we choose $\mathring{\cM}$ to be the average MDP denoted in Lemma \ref{lemma:policy aproximation}, and we obtain the regret bound of $\pi$ in all the $K$ episodes:
\begin{equation}
    \sum_{k=1}^K V_k^{\pi^*}-V_k^\pi=K \bracket{\mathring{V}^{\pi^*}-\mathring{V}^{\pi}}=K \bracket{\mathring{V}^*-\mathring{V}^\pi}\leq K\epsilon'\,.
\end{equation}
The proof is finished by taking $\epsilon'=1/K$ in Equation \eqref{eq:size}.

\end{proof}

\end{document}